\documentclass[11pt,reqno]{amsart}
\usepackage[top=1.2in, bottom=1.2in, left=1.2in, right=1.2in]{geometry}
\usepackage{amsfonts, amssymb, amsmath, mathtools, dsfont}
\usepackage[linktoc=page, colorlinks, linkcolor=blue, citecolor=blue]{hyperref}
\usepackage{enumerate, commath}

\usepackage{graphicx}

\usepackage{subcaption} 
\usepackage[font=small]{caption} 
\numberwithin{equation}{section}

\DeclareMathOperator{\E}{\mathbb{E}}

\DeclareMathOperator{\sgn}{sgn}

\DeclareMathOperator{\Span}{span}

\DeclareMathOperator{\conv}{conv}
\DeclareMathOperator{\sign}{sign}

\renewcommand{\Pr}[2][]{\mathbb{P}_{#1} \left\{ #2 \rule{0mm}{3mm}\right\}}

\newcommand{\ip}[2]{\langle#1,#2\rangle}
\newcommand{\Bigip}[2]{\Big\langle#1,#2\Big\rangle}

\def \P {\mathbb{P}}
\def \R {\mathbb{R}}

\def \K {\mathbb{K}}
\def \H {\mathbb{H}}

\def \EE {\mathcal{E}}

\def \FF {\mathcal{F}}

\def \a {\alpha}

\def \e {\varepsilon}

\def \l {\lambda}
\def \s {\sigma}

\def \tran {\mathsf{T}}

\def \psitwo {{\psi_2}}

\newcommand{\blue}{\textcolor{blue}}

\newtheorem{theorem}{Theorem}[section]

\newtheorem{corollary}[theorem]{Corollary}
\newtheorem{lemma}[theorem]{Lemma}

\newtheorem{definition}[theorem]{Definition}
\newtheorem{question}[theorem]{Question}

\theoremstyle{remark}

\begin{document}

\title{A Theory of Capacity and Sparse Neural Encoding}

\author{Pierre Baldi \and Roman Vershynin}
\date{\today}

\address{Department of Computer Science, University of California, Irvine}
\email{pfbaldi@uci.edu}

\address{Department of Mathematics, University of California, Irvine}
\email{rvershyn@uci.edu}

\thanks{The work of R.V. is in part supported by U.S. Air Force grant FA9550-18-1-0031. The work of P.B. is in part supported by grant ARO  76649-CS.}

\begin{abstract}
Motivated by biological considerations, we study sparse neural maps from an input layer to a target layer with sparse activity, and specifically the problem of storing $K$ input-target associations $(x,y)$, or memories, when the target vectors $y$ are sparse. We mathematically prove that $K$ undergoes a phase transition and that in general, and somewhat paradoxically, sparsity in the target layers increases the storage capacity of the map. The target vectors can be chosen arbitrarily, including in random fashion, and the memories can be both encoded and decoded by networks trained using local learning rules, including the simple Hebb rule. These results are robust under a variety of statistical assumptions on the data.
The proofs rely on elegant properties of random polytopes and sub-gaussian random vector variables. Open problems and connections to capacity theories and polynomial threshold maps are discussed. 
\end{abstract}

\maketitle


\section{Introduction}

Sparse representations of information are often observed in biological and artificial neural systems, and in other statistical systems as well. Arguments in support of sparsity range from low energy consumption in physical systems to interpretability in artificial models. Furthermore, sparsity can be an emergent properties, or it can be artificially designed, typically by including penalty functions that favor sparsity.
Here we study sparse encoding of information in neural maps and analyze their properties and possible computational advantages, particularly from a storage viewpoint. 

\subsection{Biological Sparsity}

Many examples of sparse representations 
in neurobiology are found, for instance, in the early processing stages of sensory systems,
across both sensory modalities and biological organisms. Together with a change in the activity pattern, from a dense input representation  to a sparse target representation in response to a stimulus,
one often observes also a significant expansion in the number of active neurons in the target layer.

For example, in the mouse visual system there are about 
20,000 projecting neurons in the dorsal Lateral Geniculate Nucleus (dLGN)
\cite{evangelio2018thalamocortical} whereas there are 
120,000-215,000 neurons in mouse primary visual cortex area V1,
where sparse activity is observed (\cite{srinivasan2015predicting} and references therein).
In the cat visual cortex, a 25-fold expansion is observed between the number of axons leaving V1 and the number of axons entering V1 from the LGN.
However, only 5--10\% of V1 neurons respond to any natural scene stimulus \cite{olshausen2004sparse}. 
The following additional examples are extracted from 
\cite{babadi2014sparseness}.
In the olfactory system of the fly, 
the antenna lobe comprising 50 glomeruli projects to the mushroom body containing about 2,500 Kenyon cells. When an odorant stimulus is presented, 59\% of the projection neurons and only 6\% of the Kenyon cells are activated \cite{turner2008olfactory}. Likewise, in rodents, the 
olfactory bulb projects to the piriform cortex 
\cite{mombaerts1996visualizing}, which hosts millions of pyramidal neurons,roughly three orders of magnitude more than the number of glomeruli in the bulb. While the response of the neurons in the olfactory bulb to odorant stimuli is quite dense \cite{vincis2012dense}, 
 only about 10\% of the neurons in the piriform cortex show an evoked response to each odorant 
 \cite{stettler2009representations,poo2009odor}.
Similar ratios are observed in the somatosensory system    
   \cite{brecht2002dynamic}, the auditory system \cite{deweese2003binary},
   and even the electrosensory system of electric fish    \cite{chacron2011efficient}.
   
The fact that the same basic strategy seems to have emerged in evolution across a variety of organisms and sensory systems requires an explanation and suggests that this strategy may have specific advantages. There have been attempts, for instance, to explain the emergence of sparse representations in V1 as reflecting the sparse, largely statistically independent components of natural images \cite{olshausen1996emergence,
bell1997independent}. However these arguments do not necessarily apply to other sensory system, or explain why a sparse basis is chosen over a dense basis that could be more compact or combinatorially richer,  or justify the expansion aspect of the strategy.

\subsection{Computational Sparsity}

On the computational side, sparsity has been studied in several different settings. 
Regularization terms, or prior distributions, associated with the L1 norm tend to produce sparsely parameterized models where a subset of the parameters are equal to zero, which can increase interpretability in some situations. The L1 approach goes back at least to work done in geology in the 1980s,
\cite{santosa1986linear} and has been further 
developed and publicized under the name of LASSO  (least absolute shrinkage and selection operator)
\cite{tibshirani1996regression} (see also 
\cite{tipping2001sparse}). Many other sparsity-related priors have been developed in recent years. 
An example of continuous ``shrinkage'' prior centered at zero is the horseshoe prior
\cite{carvalho2009handling,carvalho2010horseshoe}.
However technically these continuous priors do not have a mass at zero. Thus another alternative direction is to use discrete mixtures \cite{mitchell1988bayesian,george1993variable} where the prior on each weight $w_i$ consists of a mixture of a point mass at $w_i=0$ with an absolutely continuous distribution. A similar approach, applied to pixel intensities, rather than weights, has been developed recently to construct effective generative models of very sparse images \cite{lu2020sarm}.
Finally, there is a significant literature in compressed sensing research, where efficient sparse coding algorithms have been developed for recovering sparse signals that underwent linear compression 
\cite{donoho2006compressed,candes2006compressive,rozell2008sparse, eldar2012compressed,ganguli2012compressed,foucart2013invitation,plan2013robust,
ai2014one,plan2014dimension,plan2016generalized}.

Our main goal in this work is to better understand the computational role of sparsity in neuronal maps. Our work is closest in spirit to \cite{babadi2014sparseness}, but with a number of significant differences. First, although we discuss expansion issues, our primary focus here is on sparsity, not on expansion. Second our goal is to understand the possible computational advantages of sparsity. And Tthird, our approach is mathematical and aimed at deriving precise theorems, as opposed to approximate results derived using physics approximations or computer simulations.

\section{Basic Framework and Notation}
\label{sec:basic}

\subsection {Neural Maps and Threshold Functions} We wish to understand neural mappings $F$ from a layer of size $n$ to a layer of size $m$.
For simplicity, we call the layer of size $n$ the {\it input layer}, and the layer of size $m$  the {\it target layer} and the resulting architecture an $A(n,m)$ architecture. 
The mapping is to be implemented by $m$ linear threshold functions--as one of the simplest neuronal models--although we will briefly consider other computational units, such as polynomial threshold functions of low degree
\cite{baldi2021deep}. We let 
$\mathcal{T}(n,m)$ denote the set of all such linear threshold maps, and 
$\mathcal{T}^d(n,m)$ denote the set of all such threshold maps of degree $d$.
As a result, the activities in the target layers are always binary with value 0 or 1.
When the activities in the layer of size $n$ are also binary with values in $\{0,1\}$ or 
$\{-1,+1\}$, the units in the layer of size $m$ implement Boolean functions and we call them linear threshold gates, or polynomial threshold gates in the polynomial case.
We let $\H^n=\{0,1\}^n$ denote the $n$-dimensional hypercube with individual coordinates in $\{0,1\}$. It is sometimes more convenient to consider input vectors in   
$\K^n=\{-1,1\}^n$, the $n$-dimensional hypercube with individual coordinates in $\{-1,1\}$. A simple affine transformation transforms one type of hypercube into the other, and such transformations can be absorbed into the weights of the threshold functions, so any result obtained with a threshold map applied to input vectors in $\H^n$ can be transformed into an equivalent result with input vectors in $\K^n$ and vice versa.

\subsection{Input and Target Models}
\label{s: models}
In general, we imagine that the input layer is presented with dense input vectors $x$ of length $n$, and we want to explore their mapping into sparse representations $y$ of length $m$ in the target layer. To generate dense input vectors $x$, one can consider different models, in both the continuous and binary cases, including the following ones:
\begin{enumerate}
\item Gaussian Model $[{ N}(0,1)]^n$ in which the components of $x$ are independent identically distributed with standard normal distribution.
\item Uniform Model $U[S(n-1)]$ in which $x$ is sampled uniformly over the unit sphere in $n$-dimensional Euclidean space.
\item Bernoulli Model  $[{ B}(\frac{1}{2})]^n$ in which the components of $x$ are independent identically distributed with symmetric Bernoulli coin flip distribution with parameter p=0.5.
\item Uniform Model ${U}(\frac{1}{2},n)$ which corresponds to a uniform distribution over all vectors of length $n$ containing $n/2$ ones and $n/2$ zeros. The fact that $n$ may be odd is not important for our considerations (in this case use the the floor and ceiling operators).
\end{enumerate}
Some of the same notation and models can be used also to generate sparse vectors, so that we let:
\begin{enumerate}
\item Sparse Bernoulli Model $[{B}(p]^n$ in which the components of $x$ are independent and identically distributed with 
probability $p$ of being one (and zero otherwise), with $p$ small.
\item 
Sparse Uniform: ${U}(p,n)$ in which $x$ is sampled uniformly over the binary vectors
of $\H^n$ having a fraction 
$p$ of their entries equal to one, and the rest equal to zero, with $p$ small. 
There are of course ${n \choose np}$ such vectors, with the same remark as above regarding the use of the floor ceiling operators when $np$ is not an integer.
\end{enumerate}
Although these sparse models can also be applied to the input layer, they are meant to be applied primarily to the target layer, replacing $n$ with $m$, and $x$ with $y$.
While for certain mathematical considerations one model may be easier to use than the other, it is well known that for many probabilistic considerations, especially in terms of asymptotic results, the corresponding Bernoulli and Uniform models are very similar and that 
 $[{B}(p]^n$ is a slightly ``smeared'' version of 
${U}(p,n)$. In particular, all the vectors with $pn$ components equal to one have the same probability in  $[{B}(p)]^n$, but this probability is slightly lower compared to the corresponding uniform model due to the smearing.
Most importantly, we will also consider models, other than the uniform models, where the components of $x$ or $y$ are not independent of each other, or where $x$ and $y$ are not independent of each other.

Whatever the model, in the end we assume that we have a set of memories, or training set, consisting of $K$ pairs $(x,y)$, and one of our main goals is to find the maximal number $K$ of memories that can be stored in the neural map. 

A Boolean vector of size $n$ is called $p$-sparse if it contains $pn$ ones, and $n(1-p)$ zeros. 
Likewise, we call a Boolean function of $n$ variables $p$-sparse if its vector of assignment or targets (corresponding to the last column of its truth table)
is $p$-sparse, i.e. the function takes the value $1$ for $p2^n$ entries, and $0$ otherwise.
In general, we will use $p$ and $q$ to denote unrelated probabilities (and thus it is {\it not} the case that $q=1-p$). Finally, in order to avoid the use of double indexes, we use $x_1, x_2, ....,x_K$ to denote the set of input training examples, and  $x_1, x_2, ....,x_n$
to denote the components of an input vector $x$. Whenever this notation is used, its meaning should be obvious from the context.

\subsection{Storage and Memories} Now let us assume that we have $K$ (dense) real-valued or binary vectors $x$ in the input layer, and that we want to map them into $K$ (sparse) binary vectors, or representations, $y=F(x)$ in the target layer.
The $K$ associations $(x,y)$ are called memories and, for concreteness, the reader may think of $x$ as the activity triggered by an odor in a primary sensory layer, and of $y$ as its sparse representation in a subsequent layer.
In this work, we are concerned primarily with maximizing $K$, i.e. the number of memories that are stored in the mapping and the effects that the size
$m$ of the target layer, and the sparsity of the vectors $y$, have on the mapping. There are two additional properties of the mapping $F$ that are important: continuity and un-ambiguity. By continuous, we mean that if $x$ is one of the input memories and $x'$ is close to $x$, then in general one should expect $F(x')=F(x)$, i.e. the odors of two slightly different bananas should be mapped to identical (or very similar) binary representations. Using linear threshold functions automatically enforces continuity almost everywhere.
%
By un-ambiguity, we mean that the target vectors $y$ should be far apart from each other to avoid any possibility of confusion
(the binary representation of the banana odor should not be confused with the binary representation of the odor of any other fruit). This can be formalized for instance by maximizing the average Hamming distance between the vectors $y=F(x)$. In short we want a map $F$ that has maximal memory storage, that is also continuous and un-ambiguous. In the rest of the paper we will prove that maximizing memory storage leads to sparsity in the target layer and suggest that large target layers support un-ambiguity.

\subsection{Paradox} It must be noted from the outset that the maximization of memory storage by sparse neural maps has a paradoxical flavor. For simplicity, let us assume that we want to encode the $K$ input vectors into $K$ $p$-sparse vectors in the target layer. The total number of such possible vectors is given by $m \choose pm$ and this number is maximal when $p=0.5$. In other words there are far more possibilities for choosing the target $y$ vectors when the target vectors are constrained to be dense. 
Likewise, the total number of $p$-sparse Boolean functions of $n$ Boolean variables is given by ${2^n \choose p2^n}$, which is also maximal when $p=0.5$, providing also the impression that dense representations 
offer more choices and are easier to realize. 

\subsection{Resolution} The resolution of this paradox must come from the constraints we placed on the function $F$. In particular, consider a single linear threshold function or gate, with $K$ random input vectors of size $n$. Assume that the targets are assigned randomly with sparsity $p$. Equivalently, assume that the $K$ points are colored randomly in black and white, where $p$ is the probability of assigning a white color. When are the black and white points linearly separable?
 If $p=0.5$, 
we know \cite{baldi2019capacity} that the maximal number of random memories that can be stored satisfies $K \approx n$ (related results are known also for polynomial threshold functions \cite{baldi2019polynomial}). On the other hand, in the binary case, if only one target is equal to 1 and all the other targets are 0, it is easy to see that any $K$ memory associations can be realized, i.e. it is always possible to separate one corner of the hypercube from all the other corners using a hyperplane. Thus, in  a sense this extreme case of sparsity leads to greater storage, i.e. greater values of $K$. In short, it is intuitively clear that the smaller the fraction of white points is, the greater its chance of being linearly separable. Thus what is needed is a quantitative understanding of this phenomenon. As we are going to describe, the solution of this problem is closely related to the theory of random polytopes and is characterized by a phase transition. 

\section{Phase Transition}

We now provide a formal definition for the neural maps of interest and the underlying question we wish to address.

\begin{definition}[Threshold maps]		\label{def: threshold map}
  A map $F: \R^n \to \{0,1\}^m$ is called a {\em linear threshold map} if 
  all $m$ components of $F$ are linear threshold functions. 
  Equivalently, $F$ is a threshold map if it can be expressed as:
  $$
  F(x) \coloneqq h \big( Wx - b \big)
  $$
  for some $m \times n$ matrix $W$ and some vector $b \in \R^m$,
  where $h$ is the Heaviside function applied component-wise. The Heaviside function has value $0$ for negative arguments, and $1$ for positive arguments.
\end{definition}

Note that the bias can also be included in the weights $W$ by assuming there is one additional input unit always clamped to one. Likewise, we can define polynomial threshold maps of degree $d$ if all $m$ components of $F$ are polynomial threshold functions of degree $d$. We let $\mathcal{T}^d(n,m)$ denote the set of all such threshold maps. 
In this case, each component $i$ has the form: $f_i(x) = h(P_d(x))$ where $P_d$ is a polynomial of degree $d$ in the variables $x_1,\ldots x_n$ and $h$ is the Heaviside function. 

\begin{question}		\label{q: Tnm}
  Let $x \in \R^n$ and $y \in \{0,1\}^m$ be random vectors, possibly dependent. 
  Consider a sample of $K$ independent data points $(x_k, y_k)$ 
  drawn from the distribution of $(x,y)$.
  Does there exists a threshold map $F \in \mathcal{T}(n,m)$ such that:
  $$ 
  F(x_k) = y_k
  \quad \text{for all} \quad k=1,\ldots,K?
  $$
\end{question}


If we require $F: \R^n \to \R^m$ to be a {\em linear} map (and the distribution of $x$
is non-degenerate, e.g. absolutely continuous) then 
the answer to Question~\ref{q: Tnm} is Yes if and only if $K \le n$. 
Remarkably, for a larger class of {\em linear threshold} maps, one can fit 
samples of size much larger than $n$.

\begin{theorem}[Phase transition]			\label{thm: phase transition}
  Assume that $x$ is a standard normal random vector in $\R^n$
  and $y$ is an independent vector in $\{0,1\}^m$ whose coordinates 
  are i.i.d. Bernoulli with parameter $q \in (0,1)$. 
  Fix $\e \in (0,1)$ and let $n \to \infty$, allowing $m$, $K$ and $q$ depend on $n$.
  Assume that $K \gg n$ and $Kq \gg \log m$.
  \begin{enumerate}[1.]
    \item \label{item: phase transition small}
      If $2Kq\log(K/n) (1+\e) < n$ 
      then the answer to Question~\ref{q: Tnm} is {\em Yes}
      with probability $1-o(1)$.
    \item \label{item: phase transition large}
      If $2Kq\log(K/n) (1-\e) > n$ 
      then the answer to Question~\ref{q: Tnm} is {\em No}
      with probability $1-o(1)$.
  \end{enumerate}
\end{theorem}
Here, and everywhere else, the notation $a(n)\gg b(n)$ (or 
$b(n) \ll a(n)$) means that $b(n)/a(n) \to 0 $ as $n \to \infty$.
It is important to note how little this result depends on $m$. If we consider a single linear threshold neuron ($m=1$) corresponding to an $A(n,1)$ network, we have:

\begin{corollary}[Phase transition]			\label{thm: phase transition1}
  Assume that $x$ is a standard normal random vector in $\R^n$
  and $y$ is an independent vector in $\{0,1\}$ whose coordinates 
  are i.i.d. Bernoulli with parameter $q \in (0,1)$. 
  Fix $\e \in (0,1)$ and let $n \to \infty$, allowing $K$ and $q$ depend on $n$.
  Assume that $K \gg n$. 
  \begin{enumerate}[1.]
    \item \label{item: phase transition small1}
      If $2Kq\log(K/n) (1+\e) < n$ 
      then the sample of $K$ points is linearly separable with probability $1-o(1)$.
    \item \label{item: phase transition large1}
      If $2Kq\log(K/n) (1-\e) > n$ 
      then the sample of $K$ points is not linearly separable with probability $1-o(1)$.
  \end{enumerate}
\end{corollary}

To better understand this result, let us first notice that for $\epsilon$ very small we have:
\begin{enumerate}
\item If $K=Cn$ for some constant $C>0$, then the phase transition occurs for: 
$q=1/(2C\log C)$. For instance, if $C=10$
then the phase transition occurs for:
$q=1/(20\log 10)$.
\item If $K=n^{1+\alpha}$ for $\alpha >0$, 
then the phase transition occurs for:
$q=1/(2\alpha n^\alpha \log n)$. For instance, 
if $K=n^2$, then $\alpha =1$ and the transition occurs for: 
$q=1/(2 n \log n)$.
\end{enumerate}

Theorem~\ref{thm: phase transition} can be deduced from two results on the 
geometry of {\em Gaussian polytopes}. Consider $N$ independent random 
vectors $x_1,\ldots,x_N$ taking values in $\R^n$. Their convex hull is a random polytope in $\R^n$. 
If $x_k$ are drawn from the standard Gaussian distribution, 
the random polytope: 
$$
P \coloneqq \conv(x_1,\ldots,x_N)
$$ 
is called a Gaussian polytope.

Random polytopes including random regular polytopes \cite{affentranger1992random,
vershik1992asymptotic,boroczky1999random,
donoho2010counting}, random Gaussian 
polytopes
\cite{hug2005gaussian,barany2007central,donoho2009counting,
grote2018gaussian}, 
and more general random polytopes \cite{mendelson2005geometry,litvak2005smallest, latala2007banach, hug2013random,kabluchko2019cones, kabluchko2019expected}, have been extensively studied
in the area of asymptotic convex geometry.
One remarkable property is that random polytopes in high dimensions 
are {\em neighborly}: points $x_k$ are likely to form vertices of $P$ 
(instead of falling into the interior of $P$), 
the intervals that join pairs of points $x_k$ are likely to form edges of $P$,
the triangles that are formed by triples of points $x_k$ are likely to 
form two-dimensional faces of $P$, 
and this continues up to faces of a certain dimension $s$. D.~Donoho and J.~Tanner were 
the first to determine asymptotically sharp threshold for the critical dimension $s$ \cite{donoho2009counting}:

\begin{theorem}[Typical faces of a Gaussian polytope]		\label{thm: typical faces}
  Let $x_1,\ldots,x_N$ be independent standard Gaussian random vectors in $\R^n$. 
  Fix $\e \in (0,1)$ and let $n \to \infty$, allowing $N$ and $s$ depend on $n$. 
  \begin{enumerate}[1.]
    \item \label{item: faces 1} If $2s\log(N/n) (1+\e) < n$ then 
      $\conv(x_1,\ldots,x_s)$ is a face of the polytope $\conv(x_1,\ldots,x_N)$
      with probability $1-o(1)$ as $n \to \infty$.
    \item If $2s\log(N/n) (1-\e) > n$ then 
      $\conv(x_1,\ldots,x_s)$ is not a face of the polytope $\conv(x_1,\ldots,x_N)$
      with probability $1-o(1)$ as $n \to \infty$.
    \end{enumerate}
\end{theorem}

Motivated by the basic problem of compressed sensing, this theorem sparked many later developments, some of which are summarized in
e.g. \cite{donoho2010precise,amelunxen2014living, bayati2015universality, hug2020threshold}.
In particular, the probability in both parts of Theorem~\ref{thm: typical faces} can be improved to: 
\begin{equation}	\label{eq: faces probability}
1 - \exp(-c \e^2 s),
\end{equation}
see \cite[Theorem~II]{amelunxen2014living}. 

\begin{proof}[Proof of Part~\ref{item: phase transition small} of Theorem~\ref{thm: phase transition}.]
Let us first assume that $m=1$.
Call the points $x_k$ with labels $y_k=0$ ``black points'' and the others ``white points''. 
Let $s$ denote the number of white points. 
The assumption $Kq \gg 1$ implies that $s = Kq (1+ o(1))$ with probability $1-o(1)$. 
Let us condition on the labels $(y_k)$ with the number of white points $s$ satisfying the condition above.
Our assumption implies that:
$$
2s\log(K/n) (1+\e/2) 
\le 2Kq\log(K/n) (1+\e)
< n
$$ 
if $n$ is large.
Then, applying part \ref{item: faces 1} of Theorem~\ref{thm: typical faces}
with $\e/2$ instead of $\e$, we see that the convex hull of white points
is a face of the polytope $\conv(x_1,\ldots,x_N)$ with probability $1-o(1)$ as $n \to \infty$.
This means that the sets of black and white points are linearly separable, 
i.e. they can be separated by an affine hyperplane. Equivalently, there exists 
a threshold function $F \in \mathcal{T}(n,1)$ that realizes the data.

The general case where $m \ge 1$ follows by taking a union bound over the $m$
events, one for each coordinate of $y$. Due to \eqref{eq: faces probability}, 
the probability of failure is bounded by $m \cdot \exp(-c \e Kq) \ll 1$ if $Kq \gg \log m$.
\end{proof}

The second part of Theorem~\ref{thm: phase transition}, unfortunately, does {\em not} 
follow from Theorem~\ref{thm: typical faces} by a similar argument. 
While it is true that a set of points $x_1,\ldots,x_s$ 
that spans a face of the polytope $P = \conv(x_1,\ldots,x_N)$ must be linearly separated 
from the other points $x_{s+1},\ldots,x_N$, the converse may may not be true. 
As Figure~\ref{fig: polytope-separation} shows, points $x_1,\ldots,x_s$ might still be linearly separated from $x_{s+1},\ldots,x_N$
even if they do not form a face of $P$. 
\begin{figure}[htp]			
  \centering 
    \includegraphics[width=0.5\textwidth]{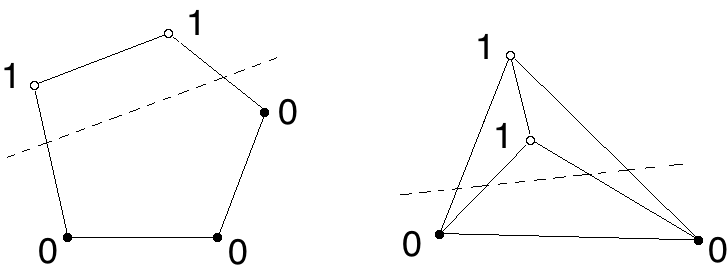} 
    \caption{Proof of Part~\ref{item: phase transition small} of Theorem~\ref{thm: phase transition}.: 
    The white points $x_k$ (labeled $y_k=1$) form a face of the Gaussian polytope 
    $\conv(x_1,\ldots,x_N)$ and thus are linearly separated from the black points. 
    However, this reasoning can not be reversed: black points may be linearly separated 
    from the white without forming a face of the Gaussian polytope.}
  \label{fig: polytope-separation}
\end{figure}

A different property of Gaussian polytopes can be used to deduce the 
second part of Theorem~\ref{thm: phase transition}: the existence of a {\em round core} of $P$.
The following result shows that $P$ contains the centered Euclidean ball of radius 
$r \approx \sqrt{2 \log(N/n)}$.

\begin{theorem}[Round core of a Gaussian polytope]		\label{thm: core}
  For every $\e \in(0,1)$ there exists $C(\e)>0$ such that the following holds.
  Assume that $N \ge C(\e) n$ and 
  let $x_1,\ldots,x_N$ be independent standard Gaussian random vectors in $\R^n$.
  Then:
  $$
  \conv(x_1,\ldots,x_N) \supset \sqrt{2 \log \Big( \frac{N}{n} \Big) (1-\e)} \cdot B(n)
  $$
  with probability at least $1-e^{-n}$.
  Here $B(n)$ denotes the unit Euclidean ball in $\R^n$ centered at the origin.
\end{theorem}

A weaker version of this result, with an absolute constant factor instead of the constant $2$, goes back to Gluskin \cite{gluskin1989extremal}, where the result is stated in the dual form. Gluskin's result inspired many further developments in the area of asymptotic convex geometry. Its ramifications can be found in particular in \cite{giannopoulos2002random,litvak2005smallest,dafnis2009asymptotic} and \cite[Section~7.5]{artstein2015asymptotic}.
None of the published versions of Gluskin's theorem, to our knowledge, exhibit the exact absolute constant $2$ that is critical for our purposes. We give a proof of Theorem~\ref{thm: core} 
in Appendix~\ref{a: core}, which essentially combines the argument in \cite{giannopoulos2002random} with an asymptotically sharp tail bound of the normal distribution.

Now we can deduce Part~\ref{item: phase transition large} of Theorem~\ref{thm: phase transition}, 
setting $m=1$ for simplicity. There are $s \approx Kq$ white points (those with labels $y_k=1$), 
and they are independent Gaussian random vectors, so their arithmetic mean $x_0$ has Euclidean norm $r_0 \approx \sqrt{n/Kq}$. 
By the assumption, this quantity is smaller than $r \approx \sqrt{2 \log(N/n)}$, 
which is the radius of the round core of the convex hull of the $N-s$ black points. So $x_0$ falls inside this round core and, as such, it is not linearly separable from the black points, see Figure~\ref{fig: core}. Hence the black and white points are not linearly separable. Here is a more formal proof.
\begin{figure}[htp]			
  \centering 
    \includegraphics[width=0.3\textwidth]{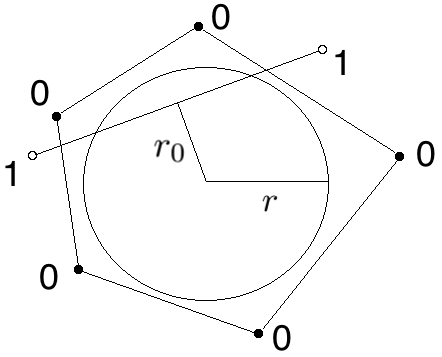} 
    \caption{Proof of Part~\ref{item: phase transition large} of Theorem~\ref{thm: phase transition}.
      The arithmetic mean of the white points (labeled $y_k=1$) has norm $r_0 \approx \sqrt{n/Kq}$. 
      This is smaller than the radius of the round core $r \approx \sqrt{2 \log(N/n)}$ of the 
      Gaussian polytope formed by the black points. Hence the black and white points are not 
      linearly separated.}
  \label{fig: core}
\end{figure}

\begin{proof}[Proof of Part~\ref{item: phase transition large} of Theorem~\ref{thm: phase transition}.]
Without loss of generality, we can assume that $m=1$. 
Condition on all labels $y_k$ 
so that the number of white points $s$ (those with labels $y_k=1$) satisfies
$s = Kq (1+ o(1))$, just like we did in the proof of the first part of the theorem. 
Without loss of generality, $q \le 1/2$. 
The number of black points $N \coloneqq K-s$ then satisfies
$N \ge K/3$ for large $n$. Thus we have for large $n$:
\begin{align} 
2s\log(N/n) 
&\ge 2Kq \log(K/n) (1-\e/2)
	\quad \text{(using also that $K \gg n$)} \nonumber\\
&\ge n (1+\e/2)
	\quad \text{(using our key assumption).} \label{eq: 2slogNn}
\end{align}

Let us apply Theorem~\ref{thm: core} for the black points and with $\e/4$ instead of $\e$.
It says that:
$$
\conv(\text{black points}) \supset r B(n)
$$
where: 
\begin{equation}	\label{eq: r large}
r = \sqrt{2 \log \Big( \frac{N}{n} \Big) \Big(1 - \frac{\e}{4} \Big)}
\ge \sqrt{\frac{n}{s} \Big( 1 + \frac{\e}{2} \Big) \Big( 1 - \frac{\e}{4} \Big)}
\ge \sqrt{\frac{n}{s} \Big( 1 + \frac{\e}{8} \Big)} 
\end{equation}
due to \eqref{eq: 2slogNn}.

On the other hand, the arithmetic mean of the white points:
$$
x_0 \coloneqq \frac{1}{s} \sum_{k:\; y_k=1} x_k 
$$
is a rescaled normal random vector, namely it $x_0 = g/\sqrt{s}$
where $g$ is a standard normal random vector in $\R^n$.
Due to a standard concentration inequality for the norm, 
$\norm{g}_2 = (1+o(1)) \sqrt{n}$ with probability $1-o(1)$, which yields:
$$
\norm{x_0}_2 = (1+o(1)) \sqrt{\frac{n}{s}}.
$$

Comparing this to \eqref{eq: r large}, we see that for large $n$:
$$
\norm{x_0}_2 < r
$$
with probability $1-o(1)$.
This means that $x_0$ lies in the ball $rB(n)$, which in turn lies in the convex hull 
of the black points. 

Summarizing, we showed that with high probability, 
the arithmetic mean of the white points $x_0$ lies in the convex hull 
of the black points. Therefore, the sets of black and white points can not be separated
by any hyperplane. Equivalently, there does not exist
any threshold function $F \in \mathcal{T}(n,1)$ that realizes the data.
The proof is complete.
\end{proof}

\subsection{Realizing all label assignments simultaneously}

The model of data considered in Theorem~\ref{thm: phase transition}, in which we assumed
that the labels $y_k$ are independent of the data points $x_k$, is not very realistic.
Fortunately, this result can be strengthened and allow for any dependence of the labels 
$y_k$ on $x_k$. The only requirement is the {\em sparsity} of label assignment.  
We say that the label assignment is {\em $s$-sparse} if,
for each coordinate $i \in \{1,\ldots,m\}$, 
at most $s$ of the labels $y_1(i),\ldots,y_K(i)$ are equal to $1$. 

\begin{theorem}[All label assignments simultaneously]			\label{thm: uniform}
  Assume that $x_1,\ldots,x_K$ are drawn independently from the 
  standard normal distribution in $\R^n$.
  Fix $\e \in (0,1)$ and let $n \to \infty$, allowing $m$, $K$ and $s$ depend on $n$.
  If: 
  $$
  2es\log \big( K/(n \cdot 2\sqrt{\pi}) \big) (1+\e) < n
  $$ 
  then the following holds with probability $1-o(1)$. 
  For any $s$-sparse label assignment $y_1,\ldots,y_K \in \{0,1\}^m$,
  there exists a function $F \in \mathcal{T}(n,m)$ such that:
  $$ 
  F(x_k) = y_k
  \quad \text{for all} \quad k=1,\ldots,K.
  $$
\end{theorem}

Up to absolute constant factors, this result is stronger than the 
first part of Theorem~\ref{thm: phase transition}.
Indeed, if $Kq \gg \log m$, the label assignment is $s$-sparse with $s = Kq(1+o(1))$ 
with probability $1-o(1)$.

Theorem~\ref{thm: uniform} follows in a way similar to the previous theorems in this Section from a stronger form of 
Donoho-Tanner's Theorem~\ref{thm: typical faces}, which was also proved in \cite{donoho2009counting}.

\begin{theorem}[All faces of a Gaussian polytope]		\label{thm: all faces}
  Let $x_1,\ldots,x_N$ be independent standard Gaussian random vectors in $\R^n$. 
  Fix $\e \in (0,1)$ and let $n \to \infty$, allowing $N$ and $s$ depend on $n$. 
  \begin{enumerate}[1.]
    \item If $2es\log \big( N/(n \cdot 2\sqrt{\pi}) \big) (1+\e) < n$ then 
      the following holds with probability $1-o(1)$ as $n \to \infty$.
      For every subset $I \subset [N]$ of size $\abs{I} \le s$, the convex hull  
      $\conv(x_i:\; i \in I)$ is a face of the polytope $\conv(x_1,\ldots,x_N)$.
    \item If $2es\log \big( N/(n \cdot 2\sqrt{\pi}) \big) (1-\e) > n$ then 
      the following holds with probability $1-o(1)$ as $n \to \infty$.
      There exists a subset $I \subset [N]$ of size $\abs{I} \le s$ such that the convex hull  
      $\conv(x_i:\; i \in I)$ is {\em not} a face of the polytope $\conv(x_1,\ldots,x_N)$.
    \end{enumerate}
\end{theorem}

Theorem~\ref{thm: phase transition} establishes the existence of a phase transition for the number $K$ of associations that can be stored in a linear threshold map, under the assumptions that $x$ is a standard normal vector and $y$ is independent from $x$. However, this leaves open two important questions. First, it would be useful to be able to prove a similar result for other realistic distributions for $x$ and $y$. It would be of particular interest to obtain results for the case where the components of $x$ are binary, or when they are not independent. And similarly for $y$, for instance when $y$ is not independent of $x$.
Second, 
Second, one would like to know if the memories  that
are plausible for a physical neural system \cite{baldi2021deep}.

These questions will be addressed using two key concepts: (1) sub-gaussian distributions;
and (2) local learning rules, in particular Hebbian learning rules. We begin by providing some background on learning rules.

\section{Learning algorithms}
Before we use sub-gaussian distributions to extend the previous theorems, it is useful to look at the algorithms by which
the memories could be learnt.  First, it should be clear that in general the $m$ units of an $A(n,m)$ architecture learn independently of each other, and thus it is enough to study learning in a single unit. 
Second, if the set of data pairs $(x,y)$ is linearly separable, the SVM learning approach of finding an hyperplane with maximum margin can be solved using linear or quadratic programming methods 
\cite{cortes1995machine,burges98,cristianini00}. However, such methods are not necessarily plausible for a physical neural system, as they do not necessarily result in a learning algorithm for the synaptic weights that is local \cite{baldi2016local}, i.e. that it depends only on variables available locally at the synapse. In practice, for the models considered here, it means that we are interested in learning rules of the form:

$$
\Delta w_i = F(x_i,y,o)
$$
Here $x_i$ is the $i$-th component of the input vector $x$, $y$ is the target value, and $o$ is the actual output value of the neuron. The rules in this section are written for a single training examples corresponding to on-line learning, but they can be averaged across multiple examples in batch learning. There are three main, different but highly related, local learning rules that can be considered:
 gradient descent, the perceptron rule, and the simple Hebb rule.

\subsection{Gradient Descent Learning Rule}  For gradient descent, we modify the Heaviside threshold function to a sigmoidal logistic function. It is well known (e.g.
\cite{baldi2021deep} that, using the relative entropy (or Kullback-Leibler divergence) between the target $y$ and the output $o$ produced by the logistic function, the gradient descent rule has the form:

$$
\Delta w_i = \eta (y-o) x_i
$$
where $\eta$ is the learning rate.
The error function is convex and therefore gradient descent, or stochastic gradient descent, with a suitable learning rate converge to 
a set of weights which minimize the error function.  

\subsection{Perceptron Learning Rule}
The perceptron learning rule \cite{rosenblatt1958perceptron}is usually written as:

$$
\Delta w_i = 
\begin{cases}
  x_i, & y=1 \;\;{\rm and}\;\; o=-1 \\
  -x_i, &y=-1\;\;{\rm and} \;\; o=+1\\
  0 & {otherwise}.
\end{cases}
\label{eq:perceptron1}
$$
using a linear threshold functions with outputs and targets in $\{-1,+1\}$. It is applied to all the weights including the bias.
The perceptron learning algorithm
initializes the weight vector to zero $w(0)=0$, and then at each step it selects an element of the training set that is mis-classified and applies the learning rule above.
Note that because the weights are initialized to zero, the learning rate simply rescales all the weights, including the biase,) and thus it can be chosen to be 1. Notice that the rule above can be rewritten as:

$$
\Delta w_i = \frac{1}{2} (y-o)x_i
$$
which shows its connection to gradient descent, and as:

$$
\Delta w_i = yx_i
$$
for the examples that are misclassified, which shows its connection to the simple Hebb rule described below.

The perceptron learning theorem states \cite{novikoff1962convergence}
that if the data is separable, then the perceptron algorithm will converge to a separating hyperplane in finite time. One may suspect that this may be the case because the rule amounts to applying stochastic gradient descent to a unit with a sigmoidal (logistic or tanh) transfer function, which is similar to a perceptron. In addition, the rule above clearly improves the performance on an example $x$ that is mis-classified. For instance if the target of $x$ is $y=+1$ and $x$ is mis-classified and selected at step $t$, then we must have $w(t)\cdot x <0$ and $w(t+1)=w(t)+x$. As a result, the performance of the perceptron on example $x$ is improved since $w(t+1) \cdot x=w(t) \cdot x + \vert \vert x \vert \vert^2$, and similarly for mis-classified examples that have a negative target. To prove convergence more precisely, it is enough to take a unit vector $w^*$ that separates the data and show that the cosine of the angle between $w(t) $ and $w^*$ increases faster than $C\sqrt t$.

\subsection{Simple Hebb Learning Rule}
The simple Hebb rule can be written as:

$$
\Delta w_i = y x_i
 $$
with a learning rate of one.
For the threshold maps $F$  considered here
(Definition~\ref{def: threshold map}), in vector form this translates to:

\begin{equation}	\label{eq: W Hebb}
W \coloneqq \sum_{k=1}^K y_k x_k^\tran
\end{equation}
 
The simple Hebb rule is the rule used, for instance, to store memories in Hopfield networks \cite{hopfield1982}
corresponding to networks of symmetrically connected linear threshold gates. As we have seen the perceptron algorithm is identical to the simple Hebb rule on the examples that are misclassified. Thus a key question to be examined is what happens when the simple Hebb rule is applied once to {\it all}
the training examples.

Thus in the next section we extend the previous sparsity results into two directions. 
First we allow more general sub-gaussian models for the data, with more complex dependency structures. Second, we show that the neural map can be implemented using the simple Hebb rule. 

\section{Computing threshold maps with sub-gaussian data and the simple Hebb rule}
\label{sec:mainth}
In a sense, Theorems~\ref{thm: phase transition} and \ref{thm: uniform}
tell us that threshold maps are able to realize memories that are completely random. 
But such memories, which lack any pattern, seem to be the hardest data to realize.
And thus one can reasonably suspect that 
threshold maps ought to be able to realize pretty much {\em any kind of data} for the same value of $K$. We are going to show that this is indeed the case. Not only any dependence of the labels $y_k$ on $x_k$ can be allowed
as we saw in Theorem~\ref{thm: uniform}, but the data points $x_k$
may come from a general distribution in $\R^n$, and without any independence 
requirements on the coordinates of $x_k$ or $y_k$. 

The reader may be quick to realize that this task is impossible in some cases, even for $m=1$. 
If the distribution of the input data consists of three points on a line, 
with the middle point labeled $1$ and the other two $0$, 
then such data is not linearly separable and thus not realizable by a linear threshold function. Remarkably, these impossible cases are rare and there is a simple recipe to rule them out. 

We only need to place standard moment assumptions on the distribution of $x$. 
Namely, we assume $x$ to be {\em sub-gaussian}, 
which means that all one-dimensional marginals of
$x$ are stochastically dominated by $\lambda g$ where $g \sim N(0,1)$ 
and $\l \ge 0$ is some number.
The smallest multiplier $\l$ can be defined as the sub-gaussian norm 
$\norm{x}_\psitwo$.
The Gaussian, Uniform, and Bernoulli models described in Section~\ref{s: models} are all examples of sub-gaussian distributions. In all of these models, the sub-gaussian norms of $x$ are bounded by an absolute constant (irrespective of $n$ or $p$).
Basic definitions about sub-gaussian distributions
are given for completeness in Appendix
\ref{a: subgauss0}, while a more extensive treatment can be found in 
\cite[Sections~2.5, 2.6, 3.4]{vershynin2018high}. 

Let us first state our result informally.

\begin{theorem}[Informal]			\label{thm: subgaussian informal}
  If $x \in \R^n$ is sub-gaussian and all coordinates of $y \in \{0,1\}^m$ take value $1$ 
  with probabilities at most $q$, and $Kq \gg \log m$, then the condition: 
  $$
  Kq \log(Km) \log(1/q) \ll n
  $$
  guarantees that all data points $(x_k,y_k)$ for which $\norm{x_k}_2 \asymp \sqrt{n}$
  can be realized by a threshold map $F$. Moreover, 
  the map $F$ can be 
 achieved using the simple Hebb rule.  
 \end{theorem}

Here and in the following sections, we use the notation $a \asymp b$ if there exist two absolute positive constants $c_1$ and $c_2$ such that $c_1b \leq a \leq c_2b$.
This notation is useful only when $a$ and $b$ vary as a function of other variables, such as $n$, and the constants are absolute in the sense that they do not depend on these other variables.
The condition $\norm{x_k}_2 \asymp \sqrt{n}$ may seem mysterious at the first sight. 
Note, however, that this condition is consistent with the natural scaling: 
if all coordinates of $x$ have unit variance, then $\E \norm{x}_2^2 = n$, so that
the norm of $x$ is expected to be of order $\sqrt{n}$. If, in addition, the coordinates of $x$
are independent, the concentration of the norm
\cite[Theorem~3.1.1]
{vershynin2018high} 
guarantees that $\norm{x}_2 \approx \sqrt{n}$ with probability $1-\exp(-cn)$. 
By a union bound, this means that the requirement $\norm{x_k}_2 \asymp \sqrt{n}$ holds automatically for all data points in the sample, so it can be removed from the statement
of the theorem.

For general distributions, however, the condition $\norm{x_k}_2 \asymp \sqrt{n}$ 
can not be removed. Jointly with the requirement of sub-gaussian distribution, this condition   
rules out the data that is impossible to realize. 
Suppose, for instance, that the distribution of $x$ is supported on a line, 
like the three-point distribution we mentioned above. Since the distribution is sub-gaussian, 
the event $\norm{x_k}_2 \asymp \sqrt{n}$ is extremely unlikely: its probability is 
exponentially small in $n$. This event is unlikely to hold for any data point in the sample. 

Let us now state Theorem~\ref{thm: subgaussian informal} formally.

\begin{theorem}[Formal]			\label{thm: subgaussian}
  Assume that $x$ is a mean zero, sub-gaussian random vector in $\R^n$, 
  and $y$ is a random vector in $\{0,1\}^m$. 
  Denote $\a \coloneqq \norm{x}_\psitwo$ and $q_i \coloneqq \Pr{y(i)=1}$, $i = 1,\ldots,m$.
  Let $m_0 \ge m$ be such that $Kq_i \ge C\log m_0$ for all $i$.
  Let $\beta,\gamma>0$ be such that:
  \begin{equation}		\label{eq: Kq subgaussian}
  C \big( \alpha^2\beta^2/\gamma^4 \big)
  Kq_i \log(Km_0) \log \frac{2}{q_i(1-q_i)} \le cn,
    \quad i = 1,\ldots,m.
  \end{equation}
  Consider $K$ data points $(x_k,y_k)$, $k=1,\ldots,K$ 
  sampled independently from the distribution of $(x,y)$.
  Then, with probability at least $1-2m/m_0$,
  there exists a map $F \in \mathcal{T}(n,m)$ such that:
  $$
  F(x_k) = y_k
  \quad \text{for all data points $x_k$ satisfying} \quad
  \gamma \sqrt{n} \le \norm{x_k}_2 \le \beta \sqrt{n}.
  $$
  Moreover, the matrix $W$ of the threshold map $F = h(Wx-b)$ can be 
  computed by the Hebb rule \eqref{eq: W Hebb} 
  and $b$ can be any vector (either fixed or dependent on the data) 
  whose coordinates $b(i)$ all satisfy:
  \begin{equation}	\label{eq: b subgaussian}
  \frac{1}{2} \gamma^2 n < b(i) < \gamma^2 n.
  \end{equation}
\end{theorem}

Note that in this theorem we do not assume any kind of independence
in the distribution of $(x,y)$. In particular, the coordinates of $x$ and of $y$ may be 
correlated with each other, and the label vector $y$ may be correlated with $x$. The proof of this theorem is given in Appendix
\ref{a: subgauss}.

\section{Binary Input Vectors}

Theorem \ref{thm: phase transition} dealt with inputs associated with a  standard normal random vectorm, and remains true for any rescaling, if the standard deviation of the normal components is not one.
From Theorem
\ref{thm: subgaussian}, we can immediately derive corollaries to deal with binary vectors
taken according to the models
$[B(p)]^n$ or $U(p,n)$ with $p=0.5$, as well as other values of $p$ (as long as $p$ is not too small). When $p=0.5$, these models are very close to the standard normal model. In the $[B(0.5)]^n$ model over $\K^n$ all the components are i.i.d. with mean zero and variance 1, as in the standard normal model. 
In the $U(0.5,n)$ model over $K^n$, all the components are identically distributed with mean zero and variance 1, and with an identical  small negative covariance for all  non-diagonal terms (see \cite{donoho2010counting} for results on randomly projected hypercubes).

\begin{corollary}[Informal]			\label{thm: subgaussian informal binary}
  If $x \in \K^n$ and all coordinates of $y \in \{0,1\}^m$ take value $1$ 
  with probabilities at most $q$, and $Kq \gg \log m$, then the condition: 
  $$
  Kq \log(Km) \log(1/q) \ll n
  $$
  guarantees that all data points $(x_k,y_k)$ 
  can be realized by a threshold map $F$. Moreover, 
  the map $F$ can be achieved using the simple Hebb rule.  
\end{corollary}

More precisely, one has the following result.

\begin{corollary}[Formal]			\label{thm: subgaussian binary}
  Assume that $x$ is a mean zero random binary vector in $\K^n$, 
  and $y$ is a random vector in $\{0,1\}^m$. 
  Denote $\a(n) \coloneqq \norm{x}_\psitwo$ and $q_i \coloneqq \Pr{y(i)=1}$, $i = 1,\ldots,m$.
  Let $m_0 \ge m$ be such that $Kq_i \ge C\log m_0$ for all $i$.
    Consider $K$ data points $(x_k,y_k)$, $k=1,\ldots,K$ 
  sampled independently from the distribution of $(x,y)$ with $K$ satisfying:
   \begin{equation}		\label{eq: Kq binary}
  C  [\alpha(n)]^2
  Kq_i \log(Km_0) \log \frac{2}{q_i(1-q_i)} \le cn,
    \quad i = 1,\ldots,m.
  \end{equation}
   
  Then, with probability at least $1-2m/m_0$,
  there exists a map $F \in \mathcal{T}(n,m)$ such that $  F(x_k) = y_k$.  
  Moreover, the matrix $W$ of the threshold map $F = h(Wx-b)$ can be 
  computed by the simple Hebb rule \eqref{eq: W Hebb}   and $b$ can be any vector (either fixed or dependent on the data) 
  whose all coordinates $b(i)$ all satisfy:
  \begin{equation}	\label{eq: b subgaussian1}
  \frac{1}{2}  n < b(i) <  n.
  \end{equation}
\end{corollary}

This corollary is obtained immediately 
by applying  Theorem
\ref{thm: subgaussian}, noting that the binary vector $x$ is sub-gaussian and that for every vector in $\K^n$:  $\vert \vert x \vert\vert=\sqrt n$. As previously stated, we know that $\alpha(n)$, which appears in  
\ref{eq: Kq binary}, is bounded.
An obvious special case of this Corollary is obtained when the components of $x$ are i.i.d. symmetric Bernoulli random variables (i.e Rademacher random variables). In Appendix 
\ref{a: subgauss0}, we show that in this case the sub-gaussian 
norm $\alpha=\alpha(n)$ is bounded by, and as $n \to \infty$ converges to, ${\sqrt 8} /{\sqrt 3}$.

\section{Input Sparsity}

Theorem~\ref{thm: subgaussian} holds for considerably general input distributions, in particular distributions that produce dense input vectors. However, one can also consider cases where the input vectors themselves tend to be sparse. In particular, this situation could occur if the first sparse target layer became the input layer for a subsequent, new, target layer.
Theorem~\ref{thm: subgaussian}
does allow the data points $x_k$ to be sparse, 
having most of their coordinates equal zero.
However, the sparsity reduces the norms of $x_k$,
thereby making the condition $\norm{x_k}_2 \asymp \sqrt{n}$ harder to satisfy, 
which in turn demands more sample points $K$ in \eqref{eq: Kq subgaussian}.

As we will show now, the data points $x_k$ can be made almost {\em arbitrarily sparse}
for free. Surprisingly, the sparsity has almost no effect on the sample size.
Let us first state this result informally. 

\begin{theorem}[Informal]			\label{thm: binomial informal}
  If $x \in \{0,1\}^n$ and $y \in \{0,1\}^m$ are independent random vectors
  whose coordinates are i.i.d and take values $1$ with probabilities 
  $p$ and $q$ respectively, 
  and $Kq \gg \log m$ and $np \gg \log(Km)$, 
  then the condition: 
  $$
  Kq \log(Km) \ll n
  $$
  guarantees that the answer to Question~\ref{q: Tnm} is {\em Yes}
  with probability $1-o(1)$.
  Moreover, the threshold map $F$ can be computed by the Hebb rule. 
\end{theorem}

And here is a formal version of the result, with more controls. 

\begin{theorem}			\label{thm: binomial}
  Assume that $x$ is a random vector in $\{0,1\}^n$ 
  and $y$ is an independent random vector in $\{0,1\}^m$. 
  Assume that the coordinates of $x$ are i.i.d. Bernoulli with parameter $p \in (0,1/2]$ 
  and the coordinates of $y$ are i.i.d Bernoulli with parameter $q \in (0,1)$. 
  Consider $K$ data points $(x_k,y_k)$, $k=1,\ldots,K$ 
  sampled independently from the distribution of $(x,y)$.
  Let $m_0 \ge m$ be such that $Kq \ge C\log m_0$, $np \ge C\log(Km_0)$, and: 
  $$
  Kq\log(Km_0) \le c n.
  $$
  Then, with probability at least $1-3m/m_0$,
  there exists a map $F \in \mathcal{T}(n,m)$ such that:
  $$
  F(x_k) = y_k
  \quad \text{for all} \quad k=1,\ldots,K.
  $$
  Moreover, the matrix $W$ in the threshold map $F = h(Wx-b)$ can be 
  computed by a version of the Hebb rule
  $W \coloneqq \sum_{k=1}^K y_k \bar{x}_k^\tran$ where $\bar{x}_k = x_k - \E x_k$, 
  and $b$ can be any vector (either fixed or dependent on the data) 
  whose all coordinates satisfy
  \begin{equation}	\label{eq: b Binomial}
  \frac{np}{4} < b(i) < \frac{np}{8}.
  \end{equation}
\end{theorem}

This result can be proved in a similar way to Theorem~\ref{thm: subgaussian}.

\begin{proof}
Let us first assume that $m=1$ and check that the map $F$ satisfies:
$$
F(x_1) = y_1
$$
with high probability. Once we have done this, a union bound
over $K$ data points and $m$ coordinates of $y$ will finish the argument. 
When $m=1$, the function $F$ can be expressed as: 
\begin{equation}	\label{eq: F Binomial}
F(x) = h \left( \ip{w}{x}-b \right)
\quad \text{where} \quad
w = \sum_{k=1}^K y_k \bar{x}_k.
\end{equation}

\medskip

{\bf Step 1. Decomposition into signal and noise.}
In order to prove that $F(x_1) = y_1$, let us expand $\ip{w}{x_1}$ as follows:
\begin{equation}	\label{eq: signal+noise}
\ip{w}{x_1}
= y_1 \ip{\bar{x}_1}{x_1} + \Bigip{\sum_{k=2}^K y_k \bar{x}_k}{x_1}
\eqqcolon \textrm{signal} + \textrm{noise}.
\end{equation}

We would like to show that the signal to noise ratio is large.
To this end, consider the random sets:
$$
I \coloneqq \{ k:\; y_k = 1 \} \subseteq \{2,\ldots,K\},
\quad
J \coloneqq \{ j:\; x_1(j) = 1 \} \subseteq \{1,\ldots,n\}.
$$
Since $y_k$ are i.i.d. Bernoulli random variables with parameter $q$,
Bernstein's inequality (see e.g. \cite[Theorem~2.8.4]{vershynin2018high}) implies that:
\begin{equation}	\label{eq: I size}
\abs{I} \le 10Kq
\end{equation}
with probability at least $1-4\exp(-c_1Kq) \ge 1-1/m_0$. 
(The last bound follows from theorem's assumption on $Kq$ with a suitably 
large constant $C$.)
Similarly, since $x_1(j)$ are i.i.d. Bernoulli random variables with parameter $q$,
we have:
\begin{equation}	\label{eq: J size}
\frac{2}{3} np \le \abs{J} \le 2np
\end{equation}
with probability at least $1-4\exp(-c_0np) \ge 1-1/(Km_0)$. 
(The last bound follows from theorem's assumption on $np$ with a suitably 
large constant $C$.)
Condition on a realization of a random vector $x_1$ and labels 
$y_2,\ldots,y_K$ satisfying \eqref{eq: J size} and \eqref{eq: I size}.

Let us estimate the strength of the signal and noise in \eqref{eq: signal+noise}.
If $y_1=0$, the signal is obviously zero, and when $y_1=1$, we have: 
$$
\textrm{signal} = \ip{\bar{x}_1}{x_1}
= \sum_{j=1}^n \left( x_1(j)-p \right) x_1(j)
= (1-p) \sum_{j=1}^n x_1(j)
= (1-p)\abs{J}
\ge \frac{np}{3}.
$$

\medskip

{\bf Step 2. Bounding the noise.}
The noise in \eqref{eq: signal+noise} can be expressed as: 
$$
\textrm{noise} = \sum_{j=1}^n \sum_{k=2}^K y_k \bar{x}_k(j) x_1(j)
= \sum_{k \in I, \, j \in J} \left( x_k(j)-p \right).
$$
The sets $I$ and $J$ are fixed by conditioning, 
and the noise is the sum of $\abs{I} \abs{J}$ i.i.d. random variables
with mean zero, variance $p(1-p)$, and which are uniformly bounded by $1$. 
Bernstein's inequality then implies that:
\begin{align*} 
\Pr{\abs{\textrm{noise}} > t \,\vert\, x_1,y_2,\ldots,y_K}
  &\le 2\exp \Big( -c_2 \min \Big\{ \frac{t^2}{\abs{I} \abs{J} p(1-p)}, t \Big\} \Big) \\
  &\le 2\exp \Big( -c_3 \min \Big\{ \frac{t^2}{Kqnp^2}, t \Big\} \Big)
  	\quad \text{(by \eqref{eq: I size} and \eqref{eq: J size})} \\
  &\le \frac{1}{Km_0}
\end{align*}
if we choose:
$$
t \coloneqq C_1 \left( \sqrt{\log(Km_0)Kqnp^2} + \log(Km_0) \right)
$$
with a suitably large constant $C_1$.
Thus, with (conditional) probability at least $1-1/(Km_0)$, 
the noise satisfies:
$$
\abs{\textrm{noise}} \le t \le \frac{np}{12}.
$$
The last bound follows from the assumptions of the theorem 
with sufficiently large constant $C$ and sufficiently small constant $c$.

\medskip

{\bf Step 3. Estimating the signal-to-noise ratio.}
Lifting the conditioning on $x_1$ and $y_2,\ldots,y_K$, we conclude the following
with (unconditional) probability at least $1-1/m_0-2/(Km_0)$. 
If $y_1=0$ then $\textrm{signal}=0$, otherwise
$\textrm{signal} \ge np/3$; 
the noise satisfies $\abs{\textrm{noise}} \le np/12$.

Putting this back into \eqref{eq: signal+noise}, we see that 
if $y_1=1$, yields: 
$$
\ip{w}{x_1} \ge \frac{np}{3} - \frac{np}{12} 
= \frac{np}{4}
> b
$$
by the assumption \eqref{eq: b Binomial} on $b$.
So $\ip{w}{x_1}-b$ is positive and thus, by \eqref{eq: F Binomial}, 
$F(x_1)=1=y_1$.

If, on the other hand, $y_1=0$ then: 
$$
\ip{w}{x_1} \le \frac{np}{12}
< b.
$$
So $\ip{w}{x_1}-b$ is negative and thus, by \eqref{eq: F Binomial}, 
$F(x_1)=0=y_1$.
Thus, in either case, we have $F(x_1) = y_1$ as claimed. 

\medskip

{\bf Step 4. Union bound.}
We can repeat this argument for any fixed $k=1,\ldots,K$
and thus obtain $F(x_k) = y_k$ with probability at least $1-1/m_0-2/(Km_0)$. 
Now take a union bound over all $k=1,\ldots,K$. 
This should be done carefully: recall that the term $1/m_0$ 
in the probability bound appears because we wanted the  
set $I$ to satisfy \eqref{eq: I size}. 
The set $I$ obviously does not depend on our choice of a particular $k$;
it is fixed during the application of the union bound and 
the term $1/m_0$ does not increase in this process. 
Thus, we showed that the conclusion:
$$
F(x_k) = y_k
\quad \text{for all} \quad k=1,\ldots,K
$$
holds with probability at least $1-1/m_0-2K/(Km_0) = 1-3/m_0$. 

This completes the proof of the theorem in the case $m=1$. 
To extend it to general $m$, we apply the  argument above 
for each coordinate $i=1,\ldots,m$ of $y$ and finish by taking the 
union bound over all $m$ coordinates.
\end{proof}

\section{Further results} 

\subsection{Autoencoders}

It is easy to check that the conclusion of Theorem~\ref{thm: binomial informal} remains the same 
if we center the label vectors $y_k$ in Hebb rule, i.e. set:
$$
W \coloneqq \sum_{k=1}^K \bar{y}_k \bar{x}_k^\tran, 
\quad \text{where} \quad
\bar{x}_k \coloneqq x_k - \E x,
\quad 
\bar{y}_k \coloneqq y_k - \E y.
$$
One can check that the effect of the centering of $y_k$ on the signal-to-noise ratio is negligible; 
we skip verifying the routine details. 

This version of Hebb rule is symmetric, so we can apply Theorem~\ref{thm: binomial informal} 
again, swapping $x_k$, $n$ and $p$ with $y_k$, $m$ and $q$ respectively. 
It follows that $F$ can be {\em inverted on the data}, and the inverse is again a threshold function!
Moreover, the inverse:
$$
F^{-1} : y_k \mapsto x_k
$$ 
is given by the same Hebb rule (up to the swapping), namely:
$$
W^\tran = \sum_{k=1}^K \bar{x}_k \bar{y}_k^\tran.
$$
This, of course, holds under the mild assumptions that 
$Kq \gg \log m$, $np \gg \log(Km)$, 
$Kp \gg \log n$, $mq \gg \log(Kn)$, as well as the key assumptions:
$$
Kq \log(Km) \ll n 
\quad \text{and} \quad
Kp \log(Kn) \ll m.
$$

This observation has an unusual consequence for ``Hebb networks'', i.e. 
two-layer neural networks whose 
weights are trained by the Hebb rule.
If we feed $x_k$ into the input layer, the network computes $y_k$ in the output layer.
Furthermore, we can reverse the direction of this computation by 
feeding $y_k$ into the output layer; the network then computes $x_k$ 
in the input layer.

One can interpret this as a construction of a ``Hebb autoencoder''
with three layers of sizes $n$, $m$ and $n$. 
If we feed the data $x_k$ into the input layer, it is transformed into $y_k$ in the hidden layer, 
and back to $x_k$ in the output layer.
Up to logarithmic factors, we can train such autoencoder to a zero error
on a sample of size $K \sim nm$ if we set $p \sim 1/n$ and $q \sim 1/m$.

\subsection{Robustness}

Hebb rule is very robust. Indeed, we can replace the exact formula 
$W \coloneqq \sum_{k=1}^K y_k x_k^\tran$ in \eqref{eq: W Hebb} by 
its approximate version:
\begin{equation}	\label{eq: Hebb correlated}
W \coloneqq \sum_{k=1}^K y_k \tilde{x}_k^\tran
\end{equation}
where $\tilde{x_k}$ are any sub-gaussian i.i.d. random vectors in $\R^n$ whose distribution is
{\em positively correlated} with $x_k$, i.e.: 
$$
\E \ip{x_k}{\tilde{x}_k} \gtrsim c n.
$$
Our analysis of signal-to-noise ratio remains mostly the same, and 
the results modify in a natural way (the constant $c$ enters the formulas).
We skip the details.

This robustness may be useful during development and learning. In addition,
it has two other consequences.

{\bf 1. Quantization.} The weights can be updated by just three values: $-1,0,1$. 
This can be seen if we use the Hebb rule \eqref{eq: Hebb correlated} with: 
$$
\tilde{x} \coloneqq \sign(x)
$$
where the sign is applied coordinate-wise. 

{\bf 2. Sparsification.} The weight matrices associated with Hebbian learning can easily be sparsified. 
All we have to do is multiply the weights by independent Bernoulli$(\rho)$ random
variables with small $\rho$. The sparsified weights are positively correlated 
with the original weights, and thus versions of our results hold for sparse networks as well. 

\subsection{Learning}

In terms of \cite[Section~8.4]{vershynin2018high}:
We showed that the empirical risk, or in-sample risk, is $R_K(f_K^*) = 0$.
Then the {\em expected error}, or expected learning risk, is:
$$
R(f_K^*) = R(f_K^*) - R_K(f_K^*)
\le \sup_{f \in \FF} \abs{R(f) - R_K(f)}
\lesssim \sqrt{\frac{\textrm{vc}(\FF)}{K}}.
$$
(The last bound can be found in \cite[Section~8.4.4]{vershynin2018high}.)

\section{Sparsity and Expansion}

The results above show that a computational advantage of sparsity in the target layer is that it allows to increase the number of memories that can be stored in the map. However it does not say anything about the expansion often observed in the target layer. Indeed, we have already noted how little the theorems derived in the previous sections depend  on the size $m$ of the target layer. Thus is there an explanation for the expansion?

There could be many reasons behind the expansion, for instance developmental constraints. However, one obvious computational reason that may be taken into consideration is producing maps that are un-ambiguous {see Section 2}. In order to minimize the risk of ambiguity, it is reasonable to try to maximize the Hamming distance between patterns in the target layer. If we have two q-sparse binary patterns, in the target layer, their maximal Hamming distance is 2qm and it is easy to see that only a linear number of patterns can be selected so that any pair of them is at maximal Hamming distance. Thus the number of such memories must grow linearly in $m$; and  the same time it must be equal to $K$, which is significantly larger than $n$ given the results in the previous theorems. Thus maximizing the pairwise distances of the target memories leads to layer expansion where $m$ is significantly larger than $n$ in order to minimize the overlap between the encodings of different memories.

\section{Conclusion and Open Problems}

In this work, we have shown that 
neural maps with a sparse hidden layer can store more memories, and both effective coding and decoding can be achieved using the simple Hebb's learning rule.
However, many open problems remain to investigate including further tightening 
the bound of some of the theorems or obtaining
results that are not necessarily asymptotic but hold exactly in some finite regime.

\subsection{Polynomial Threshold Maps}

Superficially it may seem that the results in this work are restricted to the case of linear threshold functions or gates, but this is not the case. Similar results may hold for other classes of functions, such as polynomial threshold functions or gates of degree $d$ with the functional form:

$$
F(x) = h(\sum_{I: 1 \leq \vert I \vert \leq d} w_I x^I -b) 
\label{eq:}
$$
Here $I$ runs over all non-empty subsets of $[n]=\{1,2,\ldots,n\}$, and if $I= \{i_1,\ldots,i_k\}$ we let:
$x^I=x^{i_1} \ldots x^{i_k}$. Note that in this notation we allow only pure monomials where all the powers associated with each variable are equal to one. While the more general case can be analyzed similarly, focusing on pure monomials simplifies things and furthermore, when $x \in \K^n$, $x_i^2=1$ for every $i=1, \ldots,n$ and thus higher power terms are not needed. Note also that the bias $b$ correspond to $I=\emptyset$.
We call homogeneous the case where all the monomials have degree exactly $d$: 

$$
F(x) = h(\sum_{I: \vert I \vert = d} w_I x^I -b) 
\label{eq:}
$$
For a given $n$-dimensional vector $x$, we let 
$x^{\otimes d}$ denote the tensor of all the monomials of order exactly $d$, and 
$x^{\otimes \leq d}$ denote the tensor of all non-constant monomials of order $d$ or less. Thus a polynomial threshold function (or gate), can be viewed as a linear function (or gate) applied to the corresponding tensors. 

Next, consider that the vector $x$ is a random vector with i.i.d. symmetric Bernoulli components. 
Note that in this case $x^I$ is also a symmetric Bernoulli random variable for any non-empty $I \subset [n]$.
Furthermore, for any pair of distinct subsets 
$I$ and $J$ the variables   $x^I$ and $x^J$ are
independent, i.e. there is pairwise independence but not global independence. 
Using the results from Section~\ref
{sec:mainth} leads to the following corollaries, stated first informally and then more formally. 

\begin{corollary}[Informal]			\label{thm: subgaussian informal polynomial}
  If $x \in \K^n$ has i.i.d symmetric Bernoulli components and all coordinates of $y \in \{0,1\}^m$ take value $1$ 
  with probabilities at most $q$, and $Kq \gg \log m$, then the condition 
  $$
  Kq \log(Km) \log(1/q) \ll {n\choose \leq d} \quad
  ({\rm resp.} \;\;  Kq \log(Km) \log(1/q) \ll {n\choose d})
  $$
  guarantees that all data points $(x_k,y_k)$ 
  can be realized by a polynomial (resp. homogenous polynomial) threshold map $F$ of degree $d$. 
\end{corollary}

\begin{corollary}[Formal]			\label{thm: subgaussian polynomial}
  Assume that $x \in \K^n$ has i.i.d symmetric Bernoulli components 
  and $y$ is a random vector in $\{0,1\}^m$. 
  Denote $\a \coloneqq \norm{x^{\otimes \leq d}}_\psitwo$ 
(resp.  $\a \coloneqq \norm{x^{\otimes d}}_\psitwo$ )  
  and $q_i \coloneqq \Pr{y(i)=1}$, $i = 1,\ldots,m$.
  Let $m_0 \ge m$ be such that $Kq_i \ge C\log m_0$ for all $i$.
    Consider $K$ data points $(x_k,y_k)$, $k=1,\ldots,K$ 
  sampled independently from the distribution of $(x,y)$ with $K$ satisfying:
   \begin{equation}		\label{eq: Kq binaryd}
  C  \alpha^2
  Kq_i \log(Km_0) \log \frac{2}{q_i(1-q_i)} \le c
  {n \choose \leq d},
    \quad i = 1,\ldots,m
  \end{equation}
or, respectively in the homogeneous case,

  \begin{equation}		\label{eq: Kq binarydh}
  C  \alpha^2
  Kq_i \log(Km_0) \log \frac{2}{q_i(1-q_i)} \le c
  {n \choose d},
    \quad i = 1,\ldots,m
  \end{equation}
Then, with probability at least $1-2m/m_0$,
  there exists a 
  polynomial (resp. homogeneous polynomial) threshold map $F$ of degree $d$ 
$F \in \mathcal{T}^d(n,m)$ such that $  F(x_k) = y_k$. 
 \end{corollary}
  
The proof of this statement is an immediate application of 
Theorem \ref{thm: subgaussian}, noting that:
(1) the tensors $x^{\otimes \leq d}$ (resp. 
$x^{\otimes d}$) are sub-gaussian; and 
(2) $\norm {x^{\otimes \leq d}}^2=
{n\choose \leq d}-1$ (resp.
 $\norm {x^{\otimes d}}^2=
{n\choose d}$). 
However, the bounds above depend on the value of $\alpha=\alpha(n,d)$, the sub-gaussian norm of the corresponding Bernoulli tensors.
Thus open problems here include estimating the value of $\alpha(n,d)$, and finding better estimates associated with the phase transition for polynomial threshold maps with $d>1$, in both the asymptotic and non-asymptotic regimes
(see additional discussion at the end of Appendix 
\ref{a: subgauss0}).

\subsection{Neuronal Capacity and Storage}

Finally, it is useful to view the results in this paper in terms of neuronal capacity, storage, and information theory where neural learning is seen as a communication process whereby information is transferred from the training data to the synaptic weights.
The amount of information that can be communicated, essentially the capacity of the channel, can be estimated into two different ways, one at each end of the channel. At the synaptic end, we can investigate how much information can be stored in the synapses and at the data end, we can investigate how much information can be extracted from the training set. The apparent paradox alluded to in Section \ref{sec:basic} is that in the case of sparse functions, information seems to decrease at the synaptic end, but to increase at the training data end. 
We now treat these questions more precisely by defining and comparing different notions of storage and capacity. 

For simplicity, we look at the $A(n,1)$ Boolean architectures, but the same ideas can be extended to other architectures, including $A(n,m)$ maps, as well as to non-Boolean cases. Thus in general we assume that we are considering a class ${\mathcal C}$ of Boolean functions of $n$ variables. Of particular interest here are the cases where the Boolean functions are linear threshold gates, and the training sets have targets that are sparse.
At the level of the class itself, we can first define the cardinal capacity.

\subsubsection{The Synaptic View: the Cardinal Capacity}
The cardinal capacity $C$ of  ${\mathcal C}$ is defined by:

$$
C({\mathcal C})=\log_2 \vert {\mathcal C}\vert$$
The capacity can be interpreted as the minimum average number of bits required to communicate  an element of $\mathcal C$ in a very long message consisting of a random sequence of elements in $\mathcal C$ taken with a random uniform distribution (which corresponds to the worst case in terms of the number of required bits). In the case of linear threshold gates, it can be viewed as the number of bits that must be ``communicated'' from the world (i.e. the training set) to the synaptic weights, and stored in the synaptic weights in order to select a specific input-output function. The set of all Boolean functions has capacity $2^n$. The set of all $p$-sparse Boolean functions has obviously a small cardinal capacity given by 
$\log_2 {2^n \choose p2^n}$. The set $ {\mathcal T}(n,1)$
of all linear threshold gates of $n$ variables has capacity $\log_2   \vert  {\mathcal T}(n,1) \vert \approx  n^2$ (\cite{baldi2019capacity}and references therein). The work presented here leads to an
interesting open question: what is the fraction of $p$ sparse Boolean functions that can be implemented by linear threshold gates? Or, equivalently, what is the fraction of linear threshold Boolean functions that are also $p$-sparse? And obviously a similar question can be posed for polynomial threshold gates of degree $d>1$.

If the linear threshold functions where to intersect the $p$ sparse Boolean functions roughly in the same way as all other Boolean functions do as a function of $p$, then one would conjecture that the number of $p$-sparse linear threshold gates is approximately given by:

$$
\vert {\mathcal T}_p(n,1) \vert \approx 2^{n^2} \frac{{2^n \choose p2^n}}{2^{2^n}}
\label{eq:}
$$

It is worth noting, that the value of $\vert {\mathcal T}_p(n,1) \vert $ is known exactly in some simple cases corresponding to the lowest values of $p$. In particular:

$$
\vert {\mathcal T}_{2^{-n}}(n,1) \vert =2^n
$$
since it is always possible to linearly separate one vertex of the hypercube from all the other vertices. Likewise,

$$
\vert {\mathcal T}_{2^{-(n-1)}}(n,1) \vert =
\frac{n2^n}{2}
$$
since two vertices can be linearly separated if and only if they are adjacent. And similarly for $p=3/2^n$ and $p=4/2^n$ (e.g. four vertices can be linearly separated if and only if they form a face).

Now we look at the other end of the communication channel, at the information contained in the data, which itself can be captured using different notions, such as the VC dimension, the discriminant dimension, and the training capacity.

\subsubsection{VC dimension} 
The VC dimension $V$ of ${\mathcal C}$ is the size of the largest set $S$ of input vectors  that can be shattered by ${\mathcal C}$:

$$
V({\mathcal C})=\max_{S \in \H^n} \vert S \vert : S \; \textnormal{can be be shattered}
$$
Thus obviously we have: $2^V \leq 2^C=\vert {\mathcal C }\vert$. In addition, the 
Sauer-Shelah lemma gives the upper bound:

$$
2^V \leq  2^C  \leq {2^n \choose \leq V}
$$
where ${2^n \choose \leq V}$ denotes the sum of 
all binomial terms of the form 
${2^n \choose k}$ with $k \leq V$.
The VC dimension of all Boolean functions of $n$ variables is $2^n$. The VC dimension of all $p$-sparse Boolean functions is $p2^n$.
The VC dimension of all linear threshold gates is $n+1$, which raises another problem:  What is the VC dimension of the set ${\mathcal T}_p(n,1)$ of all $p$-sparse linear threshold gates? 

\subsubsection {Discriminant Dimension}
The discriminant dimension $D$ of $\mathcal C$
is the size of the smallest set $S$ of input vectors that can be used to discriminate the elements of $
\mathcal C$, i.e. no two elements of $\mathcal C$ behave identically on this set $S$:

$$
D({\mathcal C})=\min_{S \in \H^n} \vert S \vert : \textnormal{no two elements of the class behave identically on}\; S
$$
To communicate a long sequence of elements of $\mathcal C$, in the worst case of a uniform distribution over $\mathcal C$, we: (1) first pay a fixed cost by communicating the minimal discriminant data set $S_{min}$; and then (2) encode each element $f$ of the sequence by the $D=\vert S_{min} \vert$ bits describing the behavior of $f$ on $S_{min}$. Thus asymptotically $D$ bits are sufficient to communicate a function in $\mathcal C$ and thus: $C\leq D$. The discriminant dimension of all Boolean functions of $n$ variables is $2^n$. The discriminant dimension of all $p$ sparse Boolean functions is also $2^n$. 
This leads to two open questions of determining the discriminant dimension for 
${\mathcal T} (n,1)$ and ${\mathcal T}_p (n,1)$. For linear threshold gates, the discriminant dimension is at least $n^2$.
If $D({\mathcal T})$ is the discriminant dimension for linear threshold functions, then one may conjecture that the number of $p$-sparse linear threshold gates is approximately given by:

$$
\vert {\mathcal T}_p(n,1) \vert \approx
{{D({\mathcal T}) \choose p D({\mathcal T})  }}
\label{eq:}
$$
assuming that in general $p$-sparse linear functions behave in a typical way on the discriminant set, i.e. are $p$-sparse on the discriminant set.

\subsubsection{Training Capacity}
The training capacity aims to measure the size of the largest training set $S$ that can be learnt/realized by a given learning system. This notion can only make sense if a distribution ${\mathcal D}$ is defined over possible training sets (otherwise the size of largest set is trivially $2^n$ for all non-empty ${\mathcal C}$.  A number of variations on the definition of training capacity are possible depending on various factors such as: (1) the assumptions on the distribution $\mathcal D$ of the training data:(2) whether one allows a fraction $\delta$ of the possible training sets to be un-realizable; and (3) whether one allows an error rate of up to $\epsilon$ on the data sets that are realizable.  Thus in general we may denote the training capacity by:

$$
K_{\delta,\epsilon}^{\mathcal D}
({\mathcal C})
=\max_{S \in \H^n,{\mathcal D}} \vert S \vert : \textnormal {with probability at least}\; 1-\delta \;  S \;\textnormal {can be be learnt with error at most} \epsilon
$$
For instance:

$$
K_{0,0}^{\mathcal U}
({\mathcal C})
=\max k : \textnormal {every input-ouput data set of size} \; k \; \textnormal {can be realized exactly}
$$
where $U$ denotes the uniform distribution. 

The distribution $D$ plays an important role.
If the inputs and the targets are i.i.d. with a symmetric Bernouilli distribution, then the training capacity of a linear threshold gate
is approximately $n$. However, if the targets are $p$ sparse, our results show that it is higher.

\appendix
\section{A round core of a Gaussian polytope: Proof of Theorem~\ref{thm: core}}	\label{a: core}

Assume that the Gaussian polytope
$$
P \coloneqq \conv(x_1,\ldots,x_N)
$$
does not contain the ball $r B(n)$, for some $r>0$. 
Then there exists a point $x_0 \in \P$ with $\norm{x_0} \le r$. 
This point must be separated from $P$ by a hyperplane, and in particular, by some face 
of the polytope $P$. 

To express this condition analytically, note that the points $x_i$ are in general position in $\R^n$ almost surely.
In particular, every subset $\{x_i :\; i \in I\}$ of $n$ points spans an affine hyperplane in $\R^n$. 
We can parametrize this hyperplane by its unit normal $h_I \in \R^n$ and an offset $a_I$,
always choosing the direction of $h_I$ so that $a_I \ge 0$. 
In other words, for every subset $I \subset
 [N]$ (where $[N] = \{1,\ldots,N\}$)
 with $\abs{I} = n$
there exist $h_I \in \R^n$ and $a_I \ge 0$ such that
$$
\Span \left( x_i:\; i \in I \right) = \left\{ x \in \R^n :\; \ip{h_I}{x} = a_I \right\}.
$$

When $x_0$ is separated from $P$ by a face of $P$, there exists a subset 
$I \subset [N]$ of size $\abs{I}=n$ such that the function $f(x) = \ip{h_I}{x}-a_I$ vanishes
on all points $x_i$, $i \in I$, takes the same sign on all other points $x_i$, 
and takes the opposite sign on $x_0$. 
In other words, one of the following two alternatives must happen:
\begin{gather}
\ip{h_I}{x_i} < a_I < \ip{h_I}{x_0} \quad \text{for all } i \in I^c; 	\label{eq: x0 far}\\
\ip{h_I}{x_i} > a_I > \ip{h_I}{x_0} \quad \text{for all } i \in I^c.	\label{eq: x0 close}
\end{gather}

Suppose \eqref{eq: x0 far} occurs. Then, since $\ip{h_I}{x_0} \le \norm{x_0}_2 \le r$, we have
\begin{equation}	\label{eq: ip small}
\ip{h_I}{x_i} < r \quad \text{for all } i \in I^c.
\end{equation}
If, alternatively, \eqref{eq: x0 close} occurs, then, since $a_I \ge 0$, we have
\begin{equation}	\label{eq: ip large}
\ip{h_I}{x_i} > 0 \quad \text{for all } i \in I^c.
\end{equation}
Summarizing, we have shown that if $P \not \supset rB(n)$, 
there exists $I \subset [N]$, $\abs{I} = n$, such that either \eqref{eq: ip small} or \eqref{eq: ip large} holds.

Fix $I$ and condition on the random vectors $x_i$, $i \in I$. This fixes the unit normal $h_I$.
Thus $\ip{h_I}{x_i}$, $i \in I^c$, are $N-n$ independent $N(0,1)$ random variables, 
and so we can compute the conditional probability
$$
\Pr{ \text{\eqref{eq: ip small} holds} } = \left( \Pr{g \le r} \right)^{N-n}, 
\quad \text{where} \quad g \sim N(0,1).
$$
Similarly, 
$$
\Pr{ \text{\eqref{eq: ip large} holds} } 
= \left( \Pr{g > 0} \right)^{N-n}
\le \left( \Pr{g \le r} \right)^{N-n}
$$
using symmetry and since $r > 0$.

Running the union bound over all subsets $I \subset [N]$, $\abs{I} = n$ and 
lifting the conditioning over $x_i$, $i \in I$, we conclude that 
\begin{equation}	\label{eq: P core}
\Pr{P \not \supset rB(n)} \le 2 \binom{N}{n} \left( \Pr{g < r} \right)^{N-n}.
\end{equation}

It remains to show that this quantity is bounded by $e^{-n}$ if we set
$$
r \coloneqq \sqrt{2 \log \Big( \frac{N}{n} \Big) (1-\e)}.
$$
To do so, we can use the following known Gaussian tail bound:
$$
\Pr{g \ge r} \ge \Big( \frac{1}{r} - \frac{1}{r^3} \Big) \cdot \frac{1}{\sqrt{2\pi}} e^{-r^2/2},
$$
which can be found in \cite[Theorem~1.4]{durrett2019probability} and \cite[Proposition~2.1.2]{vershynin2018high}.
Recall that we assume that $N \ge C(\e) n$ with $C(\e)$ suitably large. 
Thus we can make $r$ suitably large in terms of $\e$ and simplify the above bound to 
$$
\Pr{g \ge r} 
\ge \exp \Big[ -\Big(1 + \frac{\e}{2} \Big) \frac{r^2}{2} \Big]
\ge \exp \Big[ -\Big(1 - \frac{\e}{2} \Big) \log \Big(\frac{N}{n} \Big) \Big]
= \Big(\frac{n}{N} \Big)^{1-\e/2}.
$$
Then 
\begin{align*} 
\left( \Pr{g < r} \right)^{N-n}
&\le \Big[ 1 - \Big(\frac{n}{N} \Big)^{1-\e/2} \Big]^{N/2}
	\quad \text{(since we can assume that $N \ge 2n$)}\\
&\le \exp \Big[ -\frac{N}{2} \Big(\frac{n}{N} \Big)^{1-\e/2} \Big]
	\quad \text{(since $1-z \le e^{-z}$ for $z \ge 0$)}\\
&= \exp \Big[ -\frac{n}{2} \Big(\frac{N}{n} \Big)^{\e/2} \Big].
\end{align*}

Next, we can use the bound $\binom{N}{n} \le (eN/n)^n$ (see e.g. \cite[Exercise~0.0.5]{vershynin2018high}) 
and obtain
$$
\binom{N}{n} \left( \Pr{g < r} \right)^{N-n}
\le \exp \Big[ n \Big( \log \Big(\frac{eN}{n}\Big) - \frac{1}{2} \Big(\frac{N}{n} \Big)^{\e/2} \Big) \Big]
\le \frac{1}{2} \exp(-n).
$$
In the last step we used the assumption $N \ge C(\e) n$ with a suitably large $C(\e)$. 
Substitute this into \eqref{eq: P core} to complete the proof.
\qed

\section{Sub-Gaussian Distributions}
\label{a: subgauss0}

\subsection{Definition and Basic Properties.}
A random variable $X$ is sub-gaussian if it satisfies any of the following four (or five) equivalent properties. In the statements of these properties, the parameters $K_i>0$ differ from each other by at most an absolute constant factor.
\begin{enumerate}
\item The tail of $X$ is dominated by a Gaussian tail, that is: 

$$
{\P}(\vert X\vert \geq t) \leq 2 \exp 
(-t^2/K_1^2) \quad {\rm for} \; {\rm all} \quad t \geq 0
$$
\item The moments of $X$ satisfy:

$$
\vert\vert X \vert \vert_{L^p}=
{\E}\vert X \vert^p)^{1/p} \leq K_2 \sqrt p
\quad {\rm for} \; {\rm all} \quad p\geq 1
$$

\item The moment generating function of 
$X^2$ satisfies:

$$
{\E}\exp(\lambda^2X^2) \leq \exp (K_3^2 \lambda^2) \quad 
{\rm for} \; {\rm all}\; \lambda \; {\rm such} \; {\rm that}
\vert \lambda \vert \leq \frac{1}{K_3}
$$

\item
The moment generating function of $X^2$ is bounded at some point in the sense that:

$$
{\E} \exp (X^2/K_4^2) \leq 2
$$

\item Furthermore, if $E(X)=0$ then properties 1-4 are also equivalent to the following fifth property. The moment generating function of $X$ satisfies:

$$
{\E}\exp (\lambda X) \leq \exp(K_5^2 \lambda^2) \quad {\rm for} \; {\rm all} \quad \lambda \in \R
$$
\end{enumerate}
\noindent
The sub-gaussian norm of $X$, denoted by
$\vert \vert X\vert \vert_{\Psi_2}$ is defined by:

$$
\vert\vert X \vert\vert_{\Psi_2} =
\inf \left \{ 
t>0: \; {\E} (\exp(X^2/t^2)) \leq 2
\right \}
$$

A random vector $X$ in $\R^n$ is sub-gaussian if the one-dimensional marginals $<X,x>$ are sub-gaussian random variables for all 
$x \in \R^n$. The sub-gaussian norm of $X$ is defined as:

$$
\vert \vert X \vert \vert_{\Psi_2}
=
\sup_{x \in S^{n-1}}
\vert \vert <X,x> \vert\vert_{\Psi_2}
$$
where $S^{n-1}$ is the sphere of radius 1 in 
$\R^n$.

\subsection{Sub-gaussian norm of symmetric Bernoulli vectors}
In connection with Corollary 
\ref{thm: subgaussian informal binary}, we assume that $x=(x_1,\ldots,x_n)$ and the $x_i$ are i.i.d. Bernouilli $\pm1$ random variables with probability $p=0.5$. The sub-gaussian norm of $x$ is given by:

\begin{equation}
\alpha(n)=\vert \vert x \vert \vert_{\Psi 2} =
\sup_{u \in S^{n-1}} \vert \vert <u,x> \vert\vert_{\Psi_2}
=\sup_{u \in S^{n-1}} \inf_{t>0}
\{{\E} \exp(<u,x>^2/t^2) \leq 2 \}
\label{eq:binarysub1}
\end{equation}
where $S^{n-1}$ is the sphere of radius 1 in 
$\R^n$.
Now we can write:

\begin{equation}
{\E} \exp(<u,x>^2/t^2)=\frac{1}{2^n}
\sum_{x \in \K^n} \exp(<u,x>^2/t^2)
\label{eq:binarysub2}
\end{equation}
Note that for fixed $u$ the expectation is a continuous, strictly monotone, decreasing function of
$t \in (0,+\infty)$, decreasing in value from $+\infty$ to $0$. Thus the value $2$ is achieved by the expectation for a single value of $t$ and $\inf$ can be replaced by $\min$ in Equation \ref{eq:binarysub1}. The corresponding value of $t$ varies continuously as $u$ is varied over the closed set 
 $S^{n-1}$. Thus the maximum value of the corresponding $t$ is achieved on $S^{n-1}$ (at multiple points for symmetry reasons)
 and $\sup$ can be replaced by 
$\max$ in Equation \ref{eq:binarysub1}.
The following theorem provides the bound and asymptotic value of the sub-gaussian norm.

\begin{theorem}			\label{thm: subgaussian bernoulli}
Let $Z$ be a standard normal random variable $Z \sim N(0,1)$ and $x=(x_1,\ldots,x_n)$ be a vector of i.i.d. symmetric Bernoulli random variables.
Fix $u\in S^{n-1}$ and let $X=<u,x>$.
Then, for any $\s>0$, we have:
  $$
  \E \exp(\s^2 X^2/2) \le \E \exp(\s^2 Z^2/2) = \frac{1}{\sqrt{1-\s^2}}.
  $$
Furthermore, the sub-gaussian norm
$\alpha(n)$ of $x$ satisfies:

$$
\alpha(n) \leq
\frac{\sqrt 8}{\sqrt 3}
\quad {\rm and} \quad \alpha(n) \to 
\frac{\sqrt 8}{\sqrt 3}
\label{eq:alpha}
$$
as $n \to \infty$.
\end{theorem}

\begin{proof} [Proof of Theorem~\ref{thm: subgaussian bernoulli}]
The proof is based on the Chernoff bound on the moment generating function of $Z$ and $X$.

\begin{lemma}[Chernoff's bound]		\label{lem: chernoff}
  For any $\l \in \R$, we have 
  $$
  \E \exp(\l X) \le \E \exp(\l Z) = \exp(\l^2/2).
  $$
\end{lemma}
\noindent
To prove this bound, note that the identity for $Z$ is the basic formula for the moment generating function of the normal distribution. 
For $X$, we have
\begin{align*} 
\E \exp(\l X) 
&= \E \exp \Big( \sum_{i=1}^n u_i x_i \Big) \\
&= \prod_{i=1}^n \E \exp(\l u_i x_i)
		\quad \text{(by independence)} \\
&=\prod_{i=1}^n \cosh(\l u_i)
	\quad \text{(since $x_i = \pm 1$ with equal probabilities)} \\
&\le \prod_{i=1}^n \exp(\l^2 u_i^2/2)
	\quad \text{(since $\cosh(x) \le \exp(x^2/2)$ everywhere)} \\
&= \exp \Big( \sum_{i=1}^n \l^2 u_i^2/2 \Big) 
= \exp(\l^2/2)
	\quad \text{(by assumption on $u_i$)}
\end{align*}

Now to finish the proof of Theorem \ref{thm: subgaussian bernoulli}, we first note that the following identity holds for every $x \in \R$ and $\s>0$:
$$
\exp(\s^2 x^2/2) = \frac{1}{\s \sqrt{2\pi}} \int_{-\infty}^\infty e^{\l x} e^{-\l^2/2\s^2} \; d\l
$$
since each side represents the moment generating function of a $N(0,\s^2)$ random variable evaluated at $x$, i.e. $\E \exp(Yx)$ where $Y \sim N(0,\s^2)$. 
We then substitute $x=X$ and take expectation on both sides. This yields:
 
\begin{align*} 
\E \exp(\s^2 X^2/2) 
&= \frac{1}{\s \sqrt{2\pi}} \int_{-\infty}^\infty \E[e^{\l X}] \; e^{-\l^2/2\s^2} \; d\l \\
&\le \frac{1}{\s \sqrt{2\pi}} \int_{-\infty}^\infty e^{-\l^2/2} \; e^{-\l^2/2\s^2} \; d\l
	\quad \text{(by Lemma~\ref{lem: chernoff})} \\
&= \frac{1}{\s \sqrt{2\pi}} \int_{-\infty}^\infty e^{-\l^2/2b^2} \; d\l
	\quad \text{(where $b = \s/\sqrt{1-\s^2}$)} \\
&= \frac{b}{\s} = \frac{1}{\sqrt{1-\s^2}}.
\end{align*}
If we repeat the same computation for $Z$, the inequality (due to the application of Lemma~\ref{lem: chernoff}) becomes an equality and the first part of the theorem is proven. 
As a consequence, setting $\s^2 = 3/4$, we obtain: 
$$
\E \exp((3/8)X^2) \le \E \exp((3/8)Z^2) \le 2
$$
and thus:

$$
\norm{X}_\psitwo \le \norm{Z}_\psitwo \le \sqrt{8/3}.
$$
which completes the proof of Theorem
\ref{thm: subgaussian bernoulli}.
\end{proof}

Note, a naive Gaussian approximation to the exponent in Equation \ref{eq:binarysub2}, combined with a Lagrangian optimization argument showing that the corresponding maximal vectors have components of fixed magnitude $1/\sqrt n$, provides the estimate:

\begin{equation}
\alpha (n) \approx{\sqrt{\frac{1+\sqrt{1+(4(\ln 2) (n-1)/n)}}
{2 \ln 2}}
}
\approx
{\sqrt{\frac{1+\sqrt{1+(4(\ln 2)}}
{2 \ln 2}}
} 
\label{eq:binarysub4}
\end{equation}
which is fairly close to 
the true value ${\sqrt 8}/{\sqrt 3}$.

\subsection{Sug-gaussian norm of symmetric Bernoulli tensors}
Unlike the case $d=1$, here the numbers $\alpha(n,d)$ are not bounded as $n \to \infty$. 
To see this, let us allow for simplicity repetitions in the sets $I$ of indices defining the tensor. This makes the tensor $x^d$ have dimension $n^d$ (as opposed to $\binom{n}{d}$). With this in mind, for every vector $a$ in $\R^n$ we have: 
$X = <x^d,a^d> = <x,a>^d$. 
Let $a$ be the unit vector with all the same coefficients $1/\sqrt{n}$. By the Central Limit Theorem, $<x,a> \to  G$ where $G$ is $N(0,1)$. The convergence here is in distribution as $n \to \infty$. Thus:

$$E exp(X^2/t^2) -> E exp(G^{2d}/t^2) = \infty$$
for every $t>0$, as long as $d>1$. This shows that the sub-gaussian norm of $X$ is larger than $t$ (for large enough $n$). Since $t$ is arbitrary, it follows that the sub-gaussian norm of $ X$ goes to infinity. Using the same Central Limit Approximation used above, in the case of $d=1$, does not help in the case $d>1$.

\section{ Proof of Theorem~\ref{thm: subgaussian}}
\label{a: subgauss}
\medskip

Our proof of Theorem~\ref{thm: subgaussian} will be based on standard
facts about sub-gaussian distributions (see \cite{vershynin2018high}) 
and the following lemma.

\begin{lemma}[Conditioning sub-gaussian distributions]		\label{lem: subgaussian conditioned}
  Let $x$ be a sub-gaussian random vector taking values in $\R^n$.
  Then for any event $\EE$ with positive probability, we have 
  $$
  \norm{x}_{\psitwo(\cdot|\EE)} \le C \norm{x}_\psitwo \sqrt{\log \frac{2}{\P(\EE)}}.
  $$
\end{lemma}

In the statement of this lemma and thereafter, we write $\norm{x}_{\psitwo(\cdot|\EE)}$ to indicate 
that the sub-gaussian norm of $x$ is computed while conditioned on the event $\EE$. 

\begin{proof}
Taking the inner product of $x$ with a fixed unit vector,
we can reduce the problem to the case $n=1$ where $x$ is a random variable.
Furthermore, by homogeneity we can assume that $\norm{x}_\psitwo = 1$.
Then, denoting $q \coloneqq \P(\EE)$, we have 
$$
\Pr{x>t \,|\, \EE}
\le \frac{\Pr{x>t}}{\P(\EE)}
\le \frac{2e^{-ct^2}}{q}
\le 2e^{-ct^2/2}
$$
as long as 
$$
t \ge t_0 \coloneqq \sqrt{\frac{2}{c} \log \Big(\frac{1}{q}\Big)}.
$$
In the range where $t < t_0$, a trivial bound holds
$$
\Pr{x>t \,|\, \EE} \le 2 e^{-t^2/2t_0^2},
$$
because the right hand side is greater than $1$.
Combining the two bounds, we conclude by the definition of the sub-gaussian norm that
$$
\norm{x}_{\psitwo(\cdot|\EE)} 
\lesssim \max(1,t_0) 
\lesssim \sqrt{\log \frac{2}{q}}.
$$
The proof is complete.
\end{proof}

\bigskip

\begin{proof}[Proof of Theorem~\ref{thm: subgaussian}]
Let us first assume that $m=1$ and check that the map $F$ satisfies
$$
F(x_1) = y_1
$$
with high probability. Once we have done this, a union bound
over $K$ data points and $m$ coordinates of $y$ will finish the argument. 
When $m=1$, the function $F$ can be expressed as 
\begin{equation}	\label{eq: F subgaussian}
F(x) = h \left( \ip{w}{x}-b \right)
\quad \text{where} \quad
w = \sum_{k=1}^K y_k x_k.
\end{equation}

\medskip

{\bf Step 1. Decomposition into signal and noise.}
In order to prove that $F(x_1) = y_1$, let us expand $\ip{w}{x_1}$ as follows:
\begin{equation}	\label{eq: signal+noise subgaussian}
\ip{w}{x_1}
= y_1 \norm{x_1}_2^2 + \Bigip{\sum_{k=2}^K y_k x_k}{x_1}
\eqqcolon \textrm{signal} + \textrm{noise}.
\end{equation}

We would like to show that the signal to noise ratio is large.
To this end, consider the random set
$$
I \coloneqq \{ k:\; y_k = 1 \} \subseteq \{2,\ldots,K\}.
$$
Since $y_k$ are i.i.d. Bernoulli random variables with parameter $q$,
Bernstein's inequality (see e.g. \cite[Theorem~2.8.4]{vershynin2018high}) implies that
\begin{equation}	\label{eq: I size subgaussian}
\abs{I} \le 10Kq
\end{equation}
with probability at least $1-2\exp(-c_1Kq) \ge 1-1/m_0$. 
(The last bound follows from theorem's assumption on $Kq$ with a suitably 
large constant $C$.)
Condition on a realization of labels 
$y_2,\ldots,y_K$ satisfying \eqref{eq: J size}. 
Furthermore, condition on a realization of the random vector $x_1$ 
with moderate norm, namely
\begin{equation}	\label{eq: norm x1}
\gamma \sqrt{n} \le \norm{x_1}_2 \le \beta \sqrt{n}.
\end{equation}

Let us estimate the strength of the signal and noise in \eqref{eq: signal+noise subgaussian}.
If $y_1=0$, the signal is obviously zero, and when $y_1=1$, we have 
$$
\textrm{signal} = \norm{x_1}_2^2 \ge \gamma^2 n.
$$

\medskip

{\bf Step 2. Bounding the noise.}
To bound the noise term in \eqref{eq: signal+noise subgaussian}, 
let us use Lemma~\ref{lem: subgaussian conditioned}. It gives
\begin{equation}	\label{eq: xk conditioned}
\norm{x_k}_{\psitwo(\cdot|y_2,\ldots,y_K)} 
= \norm{x_k}_{\psitwo(\cdot|y_k)} 
\lesssim \alpha \sqrt{\log \frac{2}{q(1-q)}}, 
\quad k=2,\ldots,K.
\end{equation}
The equality in \eqref{eq: xk conditioned} is due to independence. 
The inequality in \eqref{eq: xk conditioned}
uses the fact that the events $\{y_k=0\}$ and $\{y_k=1\}$ have probabilities
with probabilities $1-q$ and $q$, both of which can be bounded below by $q(1-q)$.  

Using the approximate rotation invariance of sub-gaussian distributions
(see \cite[Proposition~2.6.1]{vershynin2018high}) followed by the bounds \eqref{eq: I size subgaussian} and \eqref{eq: xk conditioned}, 
we obtain
\begin{align*} 
\norm[3]{\sum_{k=2}^K y_k x_k}_{\psitwo(\cdot|y_2,\ldots,y_K)} 
  &= \norm[3]{\sum_{k \in I} x_k}_{\psitwo(\cdot|y_2,\ldots,y_K)}
  \lesssim \Bigg( \sum_{k \in I} \norm{x_k}_{\psitwo(\cdot|y_2,\ldots,y_K)}^2 \Bigg)^{1/2} \\
  &\lesssim \alpha \sqrt{Kq \log \frac{2}{q(1-q)}}.
\end{align*}
This implies that, conditioned on $x_1$ and $y_2,\ldots,y_k$,
the noise term in \eqref{eq: signal+noise subgaussian} is sub-gaussian: 
\begin{align} 
\norm{\textrm{noise}}_{\psitwo(\cdot|x_1,y_2,\ldots,y_K)} 
  &= \norm[3]{\Bigip{\sum_{k=2}^K y_k x_k}{x_1}}_{\psitwo(\cdot|x_1,y_2,\ldots,y_K)} 
  \le \norm[3]{\sum_{k=2}^K y_k x_k}_{\psitwo(\cdot|y_2,\ldots,y_K)} \cdot \norm{x_1}_2 \nonumber\\
  &\lesssim \alpha \sqrt{Kq \log \frac{2}{q(1-q)}} \cdot \beta \sqrt{n} \eqqcolon R. \label{eq: R}
\end{align}
By the definition of sub-gaussian norm, this yields the tail bound:
$$
\Pr{\abs{\textrm{noise}} > t \,\vert\, x_1,y_2,\ldots,y_K} 
\le 2\exp(-c_0 t^2/R^2)
\le \frac{1}{Km_0}
$$
if we choose 
$$
t \coloneqq C_1 \sqrt{\log(Km_0)} R
$$ 
with a suitably large constant $C_1$.
Thus, with (conditional) probability at least $1-1/(Km_0)$, 
the noise satisfies
$$
\abs{\textrm{noise}}
\le t \le \frac{1}{2} \gamma^2 n.
$$
The last bound follows from the definitions of $t$ and $R$ in \eqref{eq: R}
and the key assumption \eqref{eq: Kq subgaussian} with a suitably large constant $C$.

\medskip

{\bf Step 3. Estimating the signal-to-noise ratio.}
Lifting the conditioning on $x_1$ and $y_2,\ldots,y_K$, we conclude the following
with (unconditional) probability at least $1-1/m_0-1/(Km_0)$. 
If $y_1=0$ then $\textrm{signal}=0$, otherwise
$\textrm{signal} \ge \gamma^2 n$ as long as $x_1$ has moderate norm per \eqref{eq: norm x1}; 
the noise satisfies $\abs{\textrm{noise}} \le \frac{1}{2} \gamma^2 n$.

Putting this back into \eqref{eq: signal+noise subgaussian}, we see that 
if $y_1=1$ and $x_1$ has moderate norm per \eqref{eq: norm x1}, we have 
$$
\ip{w}{x_1} \ge \gamma^2 n - \frac{1}{2} \gamma^2 n 
= \frac{1}{2} \gamma^2 n
> b
$$
by the assumption \eqref{eq: b subgaussian} on $b$.
So $\ip{w}{x_1}-b$ is positive and thus, by \eqref{eq: F subgaussian}, 
$F(x_1)=1=y_1$.

If, on the other hand, $y_1=0$ then 
$$
\ip{w}{x_1} \le \frac{1}{2} \gamma^2 n
< b.
$$
So $\ip{w}{x_1}-b$ is negative and thus, by \eqref{eq: F subgaussian}, 
$F(x_1)=0=y_1$.
Thus, in either case, we have $F(x_1) = y_1$ as claimed. 

\medskip

{\bf Step 4. Union bound.}
We can repeat this argument for any fixed $k=1,\ldots,K$
and thus obtain $F(x_k) = y_k$ with probability at least $1-1/m_0-1/(Km_0)$. 
Now take a union bound over all $k=1,\ldots,K$. 
This should be done carefully: recall that the term $1/m_0$ 
in the probability bound appears because we wanted the  
set $I$ to satisfy \eqref{eq: I size}. 
The set $I$ obviously does not depend on our choice of a particular $k$;
it is fixed during the application of the union bound and 
the term $1/m_0$ does not increase in this process. 
Thus, we showed that the conclusion
$$
F(x_k) = y_k
\quad \text{for all} \quad k=1,\ldots,K
$$
holds with probability at least $1-1/m_0-K/(Km_0) = 1-2/m_0$. 

This completes the proof of the theorem in the case $m=1$. 
To extend it to general $m$, we apply argument above 
for each coordinate $i=1,\ldots,m$ of $y$ and finish by taking the 
union bound over all $m$ coordinates.
\end{proof}

\section*{Acknowledgement}
We would like to thank Dr. Sunil Gandhi for useful discussions on sparsity in the mouse visual system. Work in part supported by
grants NSF 1839429 and ARO 76649-CS.

\bibliography{math,nn,baldi,vershynin}
\bibliographystyle{plain}

\end{document}